\documentclass[lettersize,journal]{IEEEtran}
\usepackage{amsmath,amsfonts}
\usepackage{algorithmic}
\usepackage{array}
\UseRawInputEncoding
\usepackage[caption=false,font=normalsize,labelfont=sf,textfont=sf]{subfig}
\usepackage{textcomp}
\usepackage{stfloats}
\usepackage{url}
\usepackage{verbatim}
\usepackage{graphicx}
\usepackage{algorithm,algorithmic}
\usepackage{amsmath,amsfonts,amssymb,amsthm,version}
\usepackage[T1]{fontenc} 
\usepackage{color}

\usepackage{multirow}
\newtheorem{definition}{Definition}
\newtheorem{property}{Property}
\newtheorem{theorem}{Theorem}

\def\BibTeX{{\rm B\kern-.05em{\sc i\kern-.025em b}\kern-.08em
		T\kern-.1667em\lower.7ex\hbox{E}\kern-.125emX}}
\usepackage{balance}
\begin{document}
\title{GBFRS: Robust Fuzzy Rough Sets via Granular-ball Computing}
\author{Shuyin Xia, Xiaoyu Lian*, Binbin Sang, Guoyin Wang, Xinbo Gao 
	\thanks{X. Lian, S. Xia and X. Gao are with the Chongqing Key Laboratory of Computational Intelligence, Key Laboratory of Big Data Intelligent Computing, Key Laboratory of Cyberspace Big Data Intelligent Security, Ministry of Education, Chongqing University of Posts and Telecommunications, 400065, Chongqing, China. E-mail: lianxiaoyu724@qq.com, xiasy@cqupt.edu.cn, gaoxb@cqupt.edu.cn}
	\thanks{B. Sang and G. Wang are with the College of Computer and Information Science, the National Center for Applied Mathematics in Chongqing, Chongqing Normal University, Chongqing 401331, China. E-mail: sangbinbin	@cqnu.edu.cn, wanggy@cqupt.edu.cn.}}
\markboth{IEEE Transactions on Fuzzy Systems}%
{How to Use the IEEEtran \LaTeX \ Templates}
%

\maketitle

\begin{abstract}
Fuzzy rough set theory is effective for processing datasets with complex attributes, supported by a solid mathematical foundation and closely linked to kernel methods in machine learning. Attribute reduction algorithms and classifiers based on fuzzy rough set theory exhibit promising performance in the analysis of high-dimensional multivariate complex data. However, most existing models operate at the finest granularity, rendering them inefficient and sensitive to noise, especially for high-dimensional big data. Thus, enhancing the robustness of fuzzy rough set models is crucial for effective feature selection. Muiti-garanularty granular-ball computing,  a recent development, uses granular-balls of different sizes to adaptively represent and cover the sample space, performing learning based on these granular-balls. This paper proposes integrating multi-granularity granular-ball computing into fuzzy rough set theory, using granular-balls to replace sample points. The coarse-grained characteristics of granular-balls make the model more robust. Additionally, we propose a new method for generating granular-balls, scalable to the entire supervised method based on granular-ball computing. A forward search algorithm is used to select feature sequences by defining the correlation between features and categories through dependence functions. Experiments demonstrate the proposed model's effectiveness and superiority over baseline methods. The source codes and datasets are both available on the public link: https://github.com/lianxiaoyu724/GBFRS
\end{abstract}


\section{Introduction}
\IEEEPARstart{C}{lassic} rough sets can only handle symbolic data sets, requiring discretization of continuous and numerical attributes before data reduction, which may lead to the loss of classification information \cite{chen2020graph, hu2011robust}. Fuzzy rough sets (FRS), proposed by Dubois and Prade \cite{dubois1990rough}, provide an effective method to overcome data discretization issues and can handle continuous or numerical data sets without preprocessing. FRS is an effective tool for measuring sample uncertainty \cite{guo2017fuzzy}, characterizing uncertain data, and selecting informative features for downstream learning tasks. FRS has been applied to dimensionality reduction \cite{dai2017maximal,wang2019fuzzy,zhang2019active}, classification \cite{chen2020novel, hu2011robust}, regression analysis \cite{an2014fuzzy}, etc. In recent years, scholars have proposed various FRS-based feature selection models \cite{chen2013attribute, hu2023attribute, hu2006information, huang2023fuzzy, mac2019fuzzy, xu2023feature}, using fuzzy dependence functions to characterize the discriminative ability of feature subsets. However, this approach can lead to unstable classification performance on noisy data sets.

Noise samples often appear in the data, so this is one of the causes of data uncertainty. Generally speaking, there are two types of noise samples in data \cite{chen2020novel}. One is that the conditional attributes of the sample are abnormal (i.e., attribute noise) \cite{saez2022ances}, and the other is that the decision-making attributes of the sample are abnormal (i.e., quasi-noise) \cite{xia2018complete}. Since fuzzy rough sets are very sensitive to noise in uncertainty data, the evaluation of uncertainty is not accurate. The quality of data can be improved by improving the anti-noise performance of fuzzy rough set models \cite{an2021, an2021relative}. To improve the robustness to noise, Hu et al. \cite{hu2010soft} discussed the properties of the model and constructed a new dependency function from the model, using this function to evaluate and select features, and developed a new fuzzy rough set model called soft fuzzy rough set. Since fuzzy rough sets are very sensitive to noise in uncertainty data, the evaluation of uncertainty is not accurate. An et al. \cite{an2015data} proposed a probabilistic fuzzy rough model by considering the distribution information of the data. On this basis, fuzzy rough sets based on probabilistic particle distance are proposed to reduce the impact of noisy data, but the calculation amount is large and the relationship between sample labels within the neighborhood of the sample \cite{an2021probability} is ignored. Cornelis et al. \cite{cornelis2010ordered} proposed a mitigation method, which is to determine the membership degree of the upper and lower approximations through an ordered weighted average operator. Wang et al. \cite{wang2021new} further studied the equivalent expression of granular variable precision fuzzy rough sets based on fuzzy (co) meaning. It promotes the further development of fuzzy rough theory from a practical perspective. In addition, many studies consider dividing the sample set into different regions to improve the efficiency and robustness of feature selection. Hu et al. \cite{hu2009selecting} introduced a neighborhood rough set model to divide the sample set into a decision positive area and a decision boundary area and proposed a forward greedy strategy for searching feature subsets. This strategy minimizes the neighborhood decision error rate and accordingly minimizes the classification complexity of the selected feature subset. Wang et al. \cite{wang2016feature} used the concept of the fuzzy neighborhood to define fuzzy decision-making of samples, and introduced parameterized fuzzy relationships to characterize fuzzy information particles. A novel rough set model for feature subset selection is constructed using the relationship between fuzzy neighborhoods and fuzzy decision-making. And Wang et al. \cite{wang2017feature} used the concept of fuzzy neighborhood to define fuzzy decision-making of samples, and introduced parameterized fuzzy relationships to characterize fuzzy information particles. A novel rough set model for feature subset selection is constructed using the relationship between fuzzy neighborhoods and fuzzy decision-making to improve model performance. In addition, there are some other methods to optimize fuzzy rough sets for feature selection to improve model efficiency \cite{degang2010local, hu2009selecting,hu2006information}. Although the existing fuzzy rough set methods have achieved certain results in dealing with noise, most of them are based on the finest granularity as shown Fig. \ref{GBfrs}(a), and their efficiency is low and their robustness still needs to be improved.

\begin{figure}[!h]
		\centering
		\includegraphics[width=3.2in]{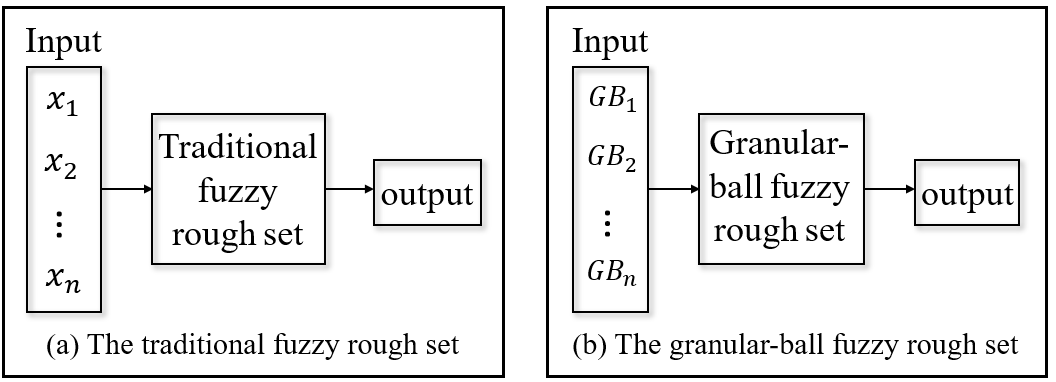}
		\caption{The comparison of traditional fuzzy rough sets and granular-ball fuzzy rough sets.}
		\label{GBfrs}
	\end{figure}

Multi-granularity computing, also known as granular computing, is a significant theoretical method in artificial intelligence. Since Lin and Zadeh proposed granular computing in 1996, scholars have increasingly studied granular information \cite{pedrycz2014allocation, yang2023multi, zhang2021double}, simulating human cognition to address complex problems \cite{hu2022multi, wang2023trilevel}. Academician Chen Lin's research results published in Sciences in 1982 pointed out that ``human cognition has the characteristic of large-scale first'': for an image, the first thing to see is the large outline, and only then will you get the specific information that makes up the outline \cite{chen1982topological}. Based on this cognitive mechanism, Wang and Xia et al. \cite{xia2019granular} proposed granular-ball computing, which divides data sets into granular-balls of different sizes, ensuring these granular-balls cover the sample and share the same characteristics. Granular-ball computing boasts robustness, efficiency, and scalability. Xia et al. \cite{xia2022gbsvm} used granular-balls generated on the data set instead of points as input to develop a granular-ball support vector machine (GBSVM) to improve classification accuracy and efficiency. In 2020, Xia and Zhang et al. \cite{xia2020gbnrs} introduced granular-ball computing into neighborhood rough sets to propose a granular-ball neighborhood rough set (GBNRS) method, which is the first to deal with continuous data without any problem-parametric rough set method. Further, Xia and Wang et al. \cite{xia2023gbrs} used granular-ball rough set (GBRS) based on granular-ball computing to simultaneously represent the unification of Pawlak rough set and neighborhood rough set. GBRS exhibits higher accuracy than traditional neighborhood rough sets and Pawlak rough sets. In addition, granular-ball computing was introduced into fuzzy concentration, and the membership degree of fuzzy granular-balls was defined and applied to the classifier SVM to propose fuzzy SVM, which improved the classification accuracy and efficiency \cite{xia2022granular}. In addition, granular-ball computing has now been effectively extended to various fields of artificial intelligence, developing theoretical methods such as granular-ball clustering methods \cite{xie2024mgnr}, granular-ball neural networks \cite{shuyin2023graph}, and granular-ball evolutionary computing \cite{xia2023granular}, significantly improving the efficiency, noise robustness, and interpretability of existing methods.

Most of the existing classic fuzzy rough set models are based on the finest granularity. The single-grained analysis mode makes most existing calculation methods inefficient and robust, and are very sensitive to noise, and noise samples often appear in classification data. Therefore, improving the robustness of fuzzy rough set models is of great significance for feature selection. Multi-granularity granular-ball computing is an important model method that has been developed in the field of granular-ball computing in recent years. This method can use ``granular-balls'' with different sizes to adaptively represent and cover the sample space, and perform learning based on the granular-balls as shown in Fig. \ref{GBfrs}(b). This paper proposes to integrate multi-granularity granular-ball computing into fuzzy rough set theory. The main contributions include three aspects:
\begin{itemize}
\item{Granular-ball computing is introduced into fuzzy rough sets, and granular-balls are used to replace sample points to propose a granular-ball fuzzy rough set (GBFRS) framework.}
\item{Within the GBFRS framework, we rigorously define the upper and lower approximations of granular-ball fuzzy rough sets,  and provide formal proofs for the related theorems, ensuring the theoretical soundness of the proposed model.}
\item{We design an algorithm process for the feature selection method based on the granular-ball fuzzy rough set. The experimental results on the UCI dataset are compared with those of existing fuzzy rough set methods, verifying the effectiveness of the proposed method.}
\end{itemize}
 
The rest of this paper is organized as follows: we introduce the concepts of fuzzy sets and the work related to granular-ball computing in section \ref{sec:related work}. Section \ref{sec3} details the granular-ball fuzzy set framework and the definition of fuzzy granular-ball. The experimental results and analysis are presented in Section \ref{sec:experiment}. Finally, some conclusions and future work are provided in Section \ref{sec6}.

\section{Related work}\label{sec:related work}

 In this section, we review some basic concepts of classic fuzzy rough sets \cite{chen2011novel, jensen2004fuzzy} and related content of granular-ball computing. 

\subsection{Basics of fuzzy rough set}

 Fuzzy rough sets can handle complex uncertainties, where two objects may have similar or identical condition attribute descriptions but belong to different decision-making classes. Assume $U$ is a domain of discourse, mapping $A \left ( \cdot  \right ): U \longrightarrow \left [ 0,1 \right ] $, then $A$ is called a fuzzy set on $U$.

\begin{definition}\label{de1}
	\setlength{\parindent}{2em}	Given $U = \{x_1, x_2,..., x_n\}$ is a set of samples and $A$ is a set of real-valued attributes. Any $a\in A$ can induce a fuzzy binary relation $R_a$ on $U$. If it satisfies, we say that $R_a$ is a fuzzy similarity relation	\\
	(1) Reflexivity: ${R}_a(x,x)=1,\forall x \in U$;	\\
	(2) Symmetry: ${R}_a(x,y)={R}_a(y,x)$, $\forall x,y\in U$.	\\
\end{definition}

Assume that $B \subseteq A$, and there exists $a\in B$. If 
$$R_{B} = \bigcap_{a \in B}R_{a}, $$
then, $R_B$ is a fuzzy similarity relation to $U$.

Suppose $D$ is a decision attribute dividing the sample set $U$ into distinct equivalence classes $U/D = \{D_1, D_2,..., D_l\}$. Then, $R_B$ can be represented by a fuzzy similarity matrix $R_{B} =\left ( r_{ij}^{B}  \right ) _{n\times n} $, where $0 \le r_{ij}^{B} \le 1, i, j = 1, 2, . . . , n$. The fuzzy similarity $r_{ij}^{B}$ is calculated by the general distance metric as follows:
 \begin{equation*} \label{up}
r_{ij}^{B} = 1-\frac{1}{\sqrt[p]{c} } \bigtriangleup _{p}^{B} \left ( x_{i}, x_{j}  \right ),  
 \end{equation*} 
where $B$ is a subset containing $m$ attributes, and the distance $\bigtriangleup_{p}^{B}$ is defined as $\bigtriangleup_{p}^{B} \left ( x_{i}, x_{j}  \right ) = \sqrt[p]{\sum_{k=1}^{ | B |} \left ( x_{ik} - x_{jk}  \right )^{p}}$. In this article, we set $p=2$, and the formula is Euclidean distance. Here, $c$ is a fixed distance parameter, which controls the value of $r_{ij}^{B}$ in the interval [0,1].

To deal with the uncertainty of decision-making attributes, fuzzy similarity relationships are used in fuzzy rough sets to describe the degree of similarity between two objects. The upper and lower approximations of any fuzzy set are defined as:
\begin{definition} 
	\setlength{\parindent}{2em}	The classic fuzzy rough set introduces the concepts of lower and upper approximation as
 \begin{equation*} \label{up}
    \begin{aligned}
\overline{R_{B}} \left ( D_{j} \right )  \left ( x \right ) &=   \max_{y\in  D_{j} } R_{B}\left ( y,x \right ), x \in U, \\
\underline{R_{B}}  \left ( D_{j} \right )  \left ( x \right ) &= \min_{y\notin  D_{j} } \left \{ 1 - R_{B}\left ( y, x \right ) \right \} , x \in U,
\end{aligned}
 \end{equation*} 
where $\overline{R_{B}} \left ( D_{j} \right )$ and $\underline{R_{B}} \left ( D_{j} \right )$ are called the fuzzy lower and upper approximations of $D_j \in \{D_1, D_2,..., D_l\}$, respectively.  
\end{definition}
$\underline{R_{B}} \left ( D_{j} \right ) \left ( x \right )$ represents the degree to which $x$ definitely belongs to $D_j$, which is equal to the minimum difference between $x$ and all samples from the sample domain $U-D_j$. And $\underline{R_{B}}  \left ( D_{j} \right )  \left ( x \right )$ is the degree to which $x$ may belong to $D_j$, which is the maximum similarity between $x$ and all samples in $D_j$.

The traditional fuzzy positive structure domain is defined as $POS_B (D) =\bigcup_{D_j \in U/D} \underline{R_{B}}(D_j)$. According to the definition of fuzzy positive domain, the fuzzy dependence function can be calculated using the following formula:
\begin{equation} \label{dependence}
\partial _{B} \left ( D \right ) = \frac{ POS_{B} \left ( D \right ) }{\left | U \right | },
 \end{equation} 
where $\left |  \cdot  \right | $ denotes the cardinality of a set. Fuzzy dependence is defined as the ratio of the size of the positive region to the total number of samples in the sample space. It serves to assess the significance of a subset of attributes.

In the implementation of the fuzzy rough reduction method, the focus is primarily on the fuzzy members of (fuzzy) equivalence classes, forming what is termed a fuzzy rough set. When all equivalence classes are crisp and contain no fuzzy members, these definitions reduce to traditional rough sets \cite{jensen2004fuzzy}. This implies that traditional rough set theory and algorithms can be directly applied to clear datasets. However, practical datasets often involve uncertainty and ambiguity, making fuzzy rough set theory more suitable for handling complex situations. This enhances the accuracy and reliability of data analysis and decision-making processes.


	\begin{figure*}[!h]
	\centering
	\includegraphics[width=5in]{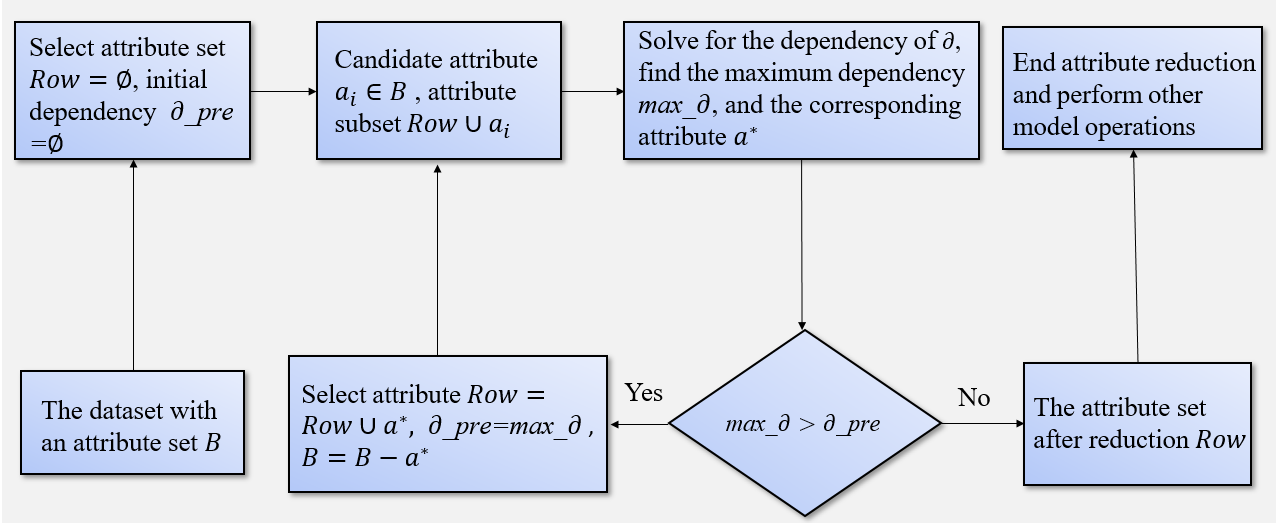}
	\caption{Attribute reduction process of fuzzy rough sets.}
	\label{Attribute}
\end{figure*} 

The lower approximation plays a crucial role in rough set theory, describing the uncertainty of decision categories under given conditions. Calculating lower approximations helps identify data relevant to the target variable under specific conditions, measuring the discriminative ability of features. The features that obtain a larger lower approximation indicate that they can effectively distinguish different decision categories, highlighting their importance in the feature selection process. In attribute selection, fuzzy rough sets evaluate the importance of each attribute through lower approximation. The attribute reduction process of fuzzy rough sets is shown in Fig. \ref{Attribute}. First, the data is preprocessed, including normalizing continuous attributes and fuzzifying discrete attributes. Next, an appropriate similarity measure calculates a similarity matrix for each attribute. Then, for each class, a fuzzy lower approximation is calculated for each object, which depends on the object's similarity to its nearest neighbors and their class membership. Finally, the importance of each attribute is evaluated by the consistency between the lower approximation and the class label, and the most important attributes are selected according to the ranking of contribution value. This process effectively filters out the most contributing attributes, enhancing classification accuracy and model interpretability. Most existing fuzzy rough set methods rely on sample points, which can be susceptible to noise, impacting both upper and lower approximation calculations and data analysis.

\subsection{Granular-ball computing}

 The granular-ball computing method primarily utilizes granular-balls of varying sizes to replace fine sample points in constructing a multi-granularity dataset that describes the distribution of the original dataset. The choice of using balls is mainly due to their concise and unified mathematical model expression in any dimensional space, requiring only the center and radius for representation.

	\begin{figure}[!h]
		\centering
		\includegraphics[width=1.5in]{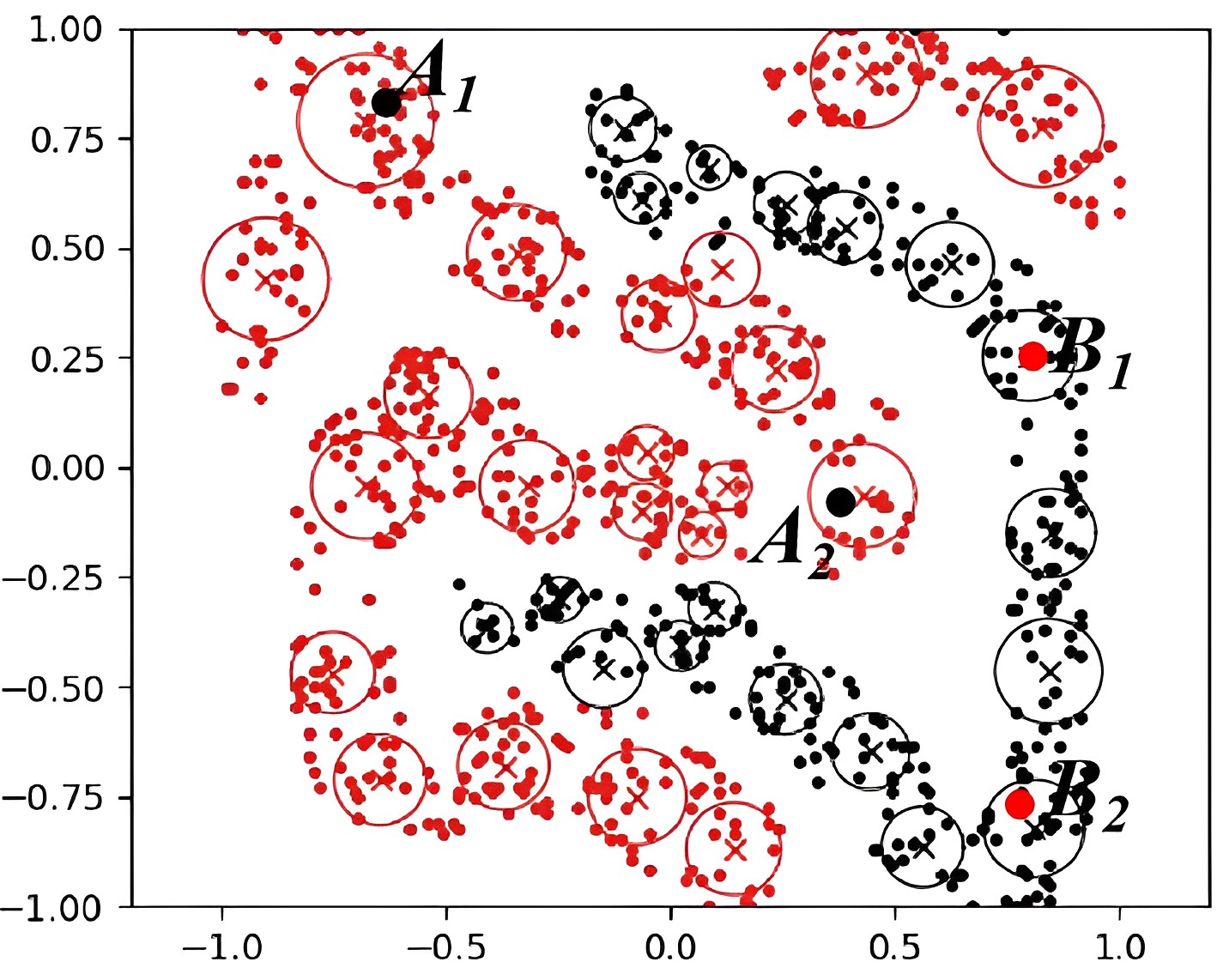}
		\includegraphics[width=1.5in]{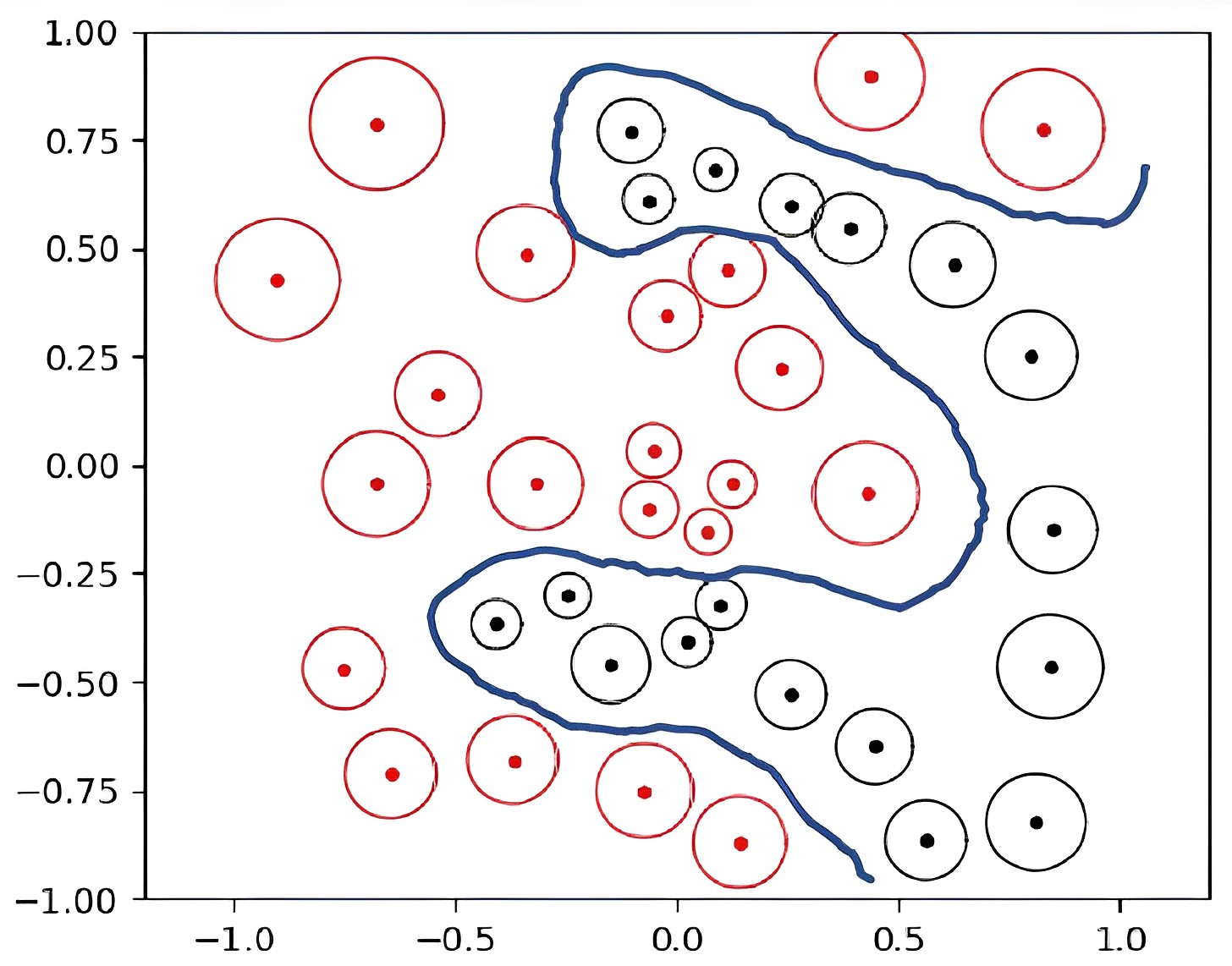}\\
		(a) \ \ \ \ \ \ \ \ \ \ \ \ \ \ \ \ \ \ \ \ \ \ \ \ \ \ \ \ (b)
		\caption{Take the data set ``fourclass'' as an example, the result of granular-ball splitting. (a) The granular-balls cover the sample set; (b) The granular-ball decision boundary is consistent with the original data.}
		\label{generateGB}
	\end{figure} 

Taking the classification problem as an example,  the representation and computation of granular-ball coverage are illustrated in Fig. \ref{generateGB}. Fig. \ref{generateGB}(a) shows the particle coverage result on the 'fourclass' dataset, while Fig. \ref{generateGB}(b) depicts the result after retaining the granular-balls. The boundaries of the balls align well with the sample boundaries, with both balls and samples marked with corresponding classification labels. The two colors represent two types of labels. Obviously, Fig. \ref{generateGB}(a) contains thousands of sample points, whereas Fig. \ref{generateGB}(b) shows nearly 50 granular-balls. This reduction by more than 20 times demonstrates the efficiency of granular-ball computing. Additionally, some red granular-balls contain a small number of black points, indicating label noise characteristics like $A_1$ and $A_2$. However, these black points do not affect the labels of coarse-grained granular-balls, highlighting the method's robustness. The granular-ball computing model can be defined as follows.

Given a data set $D = \{x_i, i = 1, 2, . . . , n\}$, where $x_i$ and $n$ represent the samples and the number of samples in D. By using granular-ball computing, a granular-ball set $G= \{GB_j, j =1, 2, . . . , k\}$ is generated, and $k$ denotes the number of all balls generated on $D$. The standard model for Multi-granularity granular-ball computing is as follows:

	\begin{equation}\label{GBmodel}
		\left\{\begin{aligned}
			&\quad \quad f(x,\vec{\alpha}) \longrightarrow  g(GB,\vec{\beta})\\
			&s.t.\quad {\mathop{\min}  } \ \  \frac{n}{\sum_{j = 1}^{k}  \left| GB_{j} \right|} +  k + loss(GB) , \\
			& \quad \quad quality (GB_{j}) \geq T, j=1,2,...,k, \\
		\end{aligned}\right.
	\end{equation}
where $\left| \cdot \right|$ denotes the number of samples in a set, i.e. the number of samples in a granular-ball. The function $f$ represents the original classification learning model that takes points $x$ as input, while function $g$ stands for the classification learning model that utilizes granular-balls $GB$ as input. $\vec{\alpha}$ and $\vec{\beta}$ denote the parameter vectors. In addition, the purity threshold $T$ controls the minimum quality of all balls. The model's constraints primarily govern granular-ball generation. The first term of the first constraint reflects the reciprocal of sample coverage by granular-balls. Greater coverage generally reduces information loss, maintaining learning performance. The second term controls the number of granular-balls; fewer balls enhance computational efficiency and robustness. The third term applies when granular-balls are used in methods like deep learning, necessitating quality control in loss functions via metric learning, due to changes in sample point representation. The second constraint is that the mass of all granular-balls must meet the threshold requirement, which controls the minimum granularity size of the granular-balls.

	\begin{figure*}[!h]
		\centering
		\includegraphics[width=5in]{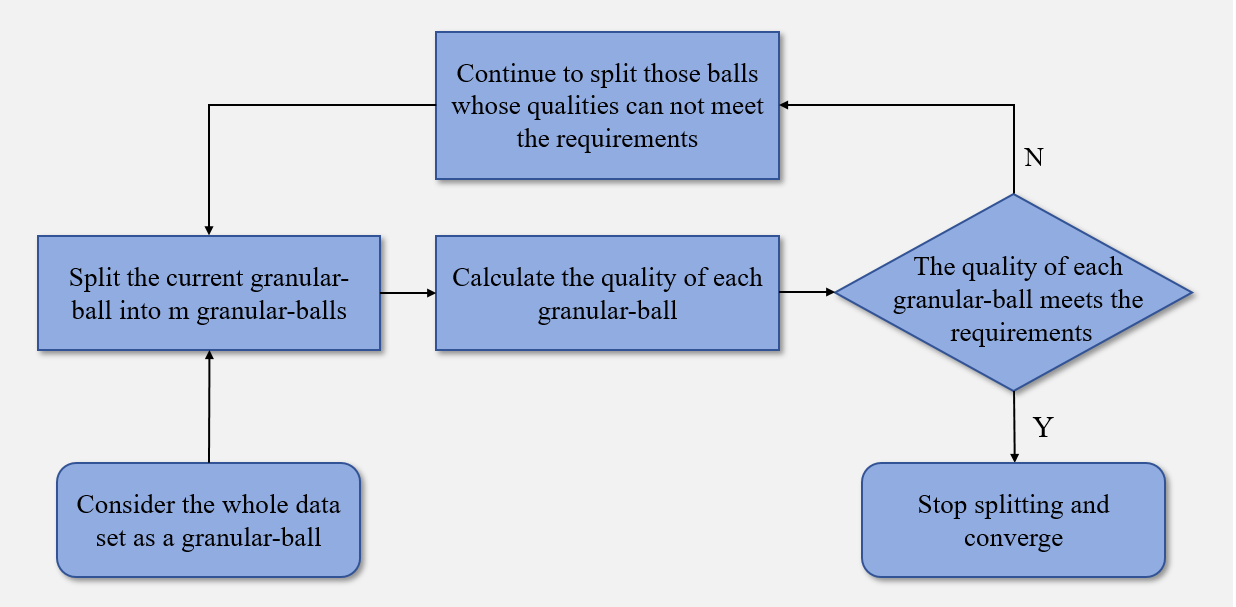}
		\caption{The flowchart of granular-balls generation.}
		\label{GB}
	\end{figure*}

The process of granular-ball generation is illustrated in Fig. \ref{GB}. Initially, the entire dataset is represented as a single granular-ball, which typically fails to meet the required purity standards. Subsequently, this ball undergoes division into m sub-balls using the k-means algorithm. The combined mass of these sub-balls is then evaluated to ascertain if they collectively satisfy the quality criteria. If not, the subdivision process continues iteratively until each granular-ball fulfills the minimum quality standards. Although the theory of granular-ball computing has not been proposed for a long time, a series of methods have been developed to address problems such as efficiency, robustness, and interpretability in various fields of artificial intelligence.

\section{Granular-ball fuzzy rough set}\label{sec3}

\subsection{Motivation}

Most of the existing fuzzy rough set methods are aimed at sample points and need to calculate similarity matrices and upper and lower approximations between all objects, involving a large number of distance calculations and matrix operations, especially on large-scale data sets, which will lead to computational complexity and time costs increase significantly. Fuzzy rough sets are also sensitive to noise points in the data, affecting the accuracy of lower approximations and, consequently, attribute importance assessment.

\begin{figure}[!h]
		\centering
		\includegraphics[width=3in]{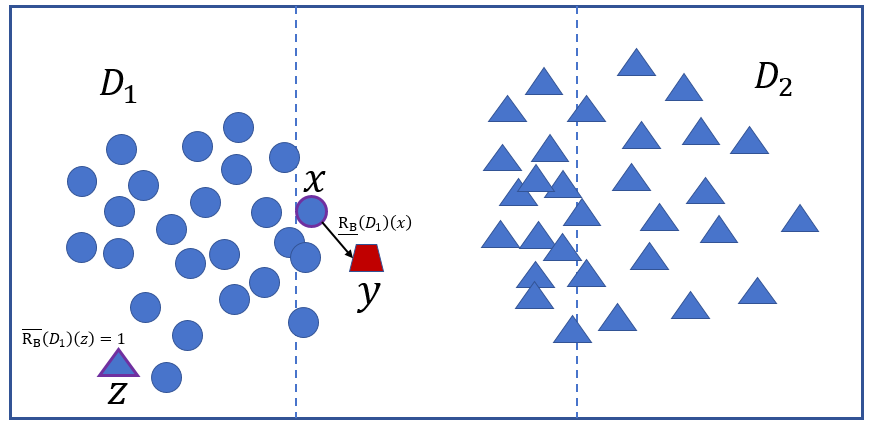}
		\caption{The illustration is that traditional fuzzy rough sets are not robust.}
		\label{frs}
	\end{figure}
 
As an example of the lower approximation, Fig. \ref{frs} shows the effect of noise points in the fuzzy rough set on the lower approximation. For a point $x$, if we want to determine the degree to which it belongs to a certain class $D_i$, it depends on the nearest heterogeneous boundary to it. In the figure, $x$ is located within class $D_1$, but there is a noise point $y$ near it. Since point $y$ is very close to $x$, it may be mistaken as the nearest heterogeneous boundary of $x$, thus affecting the lower approximation of $x$. Similarly, point $z$ actually belongs to class $D_2$, but its location is close to class $D_1$, so it is likely to be regarded as a noise point. Although point $z$ has an upper approximation of 1 (i.e. very much belongs to class $D_2$) because its nearest point is itself, it may actually belong to class $D_1$. The upper and lower approximations of fuzzy rough sets are values, while the upper and lower approximations of rough sets are set. $x$ represents an arbitrary point, and $y$ is a certain approximate point. According to the formula, the lower approximation of $x$ is the point closest to $x$ and does not belong to the $D_i$ category. In this case, noise points such as $y$ will affect the lower approximate calculation result of $x$, thereby affecting the accuracy of $x$ classification. We can see the important impact of noise points on the approximation under fuzzy rough sets, especially when calculating the degree to which a point belongs to a certain class.

\begin{figure}[!h]
		\centering
		\includegraphics[width=3in]{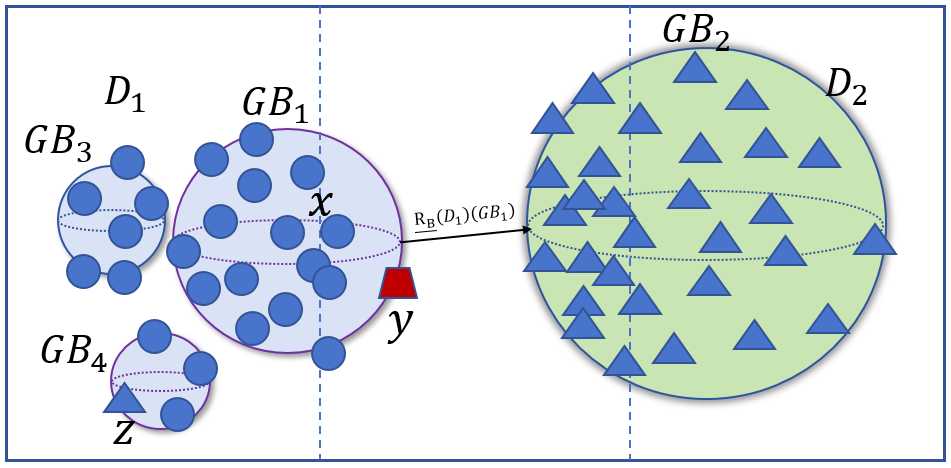}
		\caption{Schematic diagram of the robustness of granular-ball fuzzy rough sets.}
		\label{GBfrs2}
	\end{figure}

To this end, we consider introducing granular ball calculation in fuzzy rough sets, using granular-balls of different granularities as units instead of points, as shown in Fig. \ref{GBfrs2}. The label of a granular-ball is defined as the label that appears most frequently in a granular-ball. The blue balls $GB_1, GB_3, GB_4$ belong to class $D_1$, and the green balls $GB_2$ belong to class $D_2$. At this time, we analyze the granular-balls as units and solve the upper and lower approximations so that the blue granular-balls are not affected by the noise points $y$ and $z$.


\subsection{Granular-ball Fuzzy Rough Set}

GBFRS introduces granular-ball computing into fuzzy rough sets, so it must follow the granular-ball calculation model (\ref{GBmodel}). Each granular-ball consists of two parameters: center and radius, which are defined as:
\begin{definition}\label{defc}
A granular-ball $GB = \{x_i,i = 1, ... , N\}$, where $x_i$ denotes the samples in $GB$ and $N$ is the number of samples in $GB$. $GB's$ center $c$ and radius $r$ get values by:
$$
c = \frac{1}{N} \sum_{i = 1}^{N} x_i; \ \ \ r = \frac{1}{N} \sum_{i = 1}^{N}  \| x_i - c  \|,
$$	
where $\|x_i - c \|$ denotes the Euclidean distance from $x_i$ to $c$.
\end{definition}

The quality of a granular-ball is assessed based on its ``purity'', defined as the proportion of samples belonging to the dominant category within the ball. For instance, if a granular-ball contains 100 samples categorized into two classes, with 80 samples labeled as `+1' and 20 samples labeled as `-1', the purity is determined by the ratio of the dominant category's samples to the total samples within the ball. In this example, the label `+1' represents the majority with 80 samples out of 100, resulting in a purity of 0.8, thus assigning the label `+1' to the granular-ball. Granular-balls replace sample points to describe the entire data set, and it is necessary to ensure that the quality of each ball reaches the requirements.

 
 The structure of the data set for classification learning is represented as $\left\langle U, A, D \right\rangle$, where $U$ is a non-empty set of samples, and $A$ is a set of real-valued attributes. An attribute $a\in A$ is a mapping defined on $U$, that is, $a:U\rightarrow R^+\cup \{ 0 \}$. $R$ is the field of real numbers. $D$ is the decision attribute. Suppose there is a data set $U=\left\{x_i,i=0,1,...,n\right\}$, and convert all sample points into a set of granular-balls as $G=\left\{{GB}_i,i=0,1,...,m\right\}$, then the sample space is transformed into ${U}'=G$. The fuzzy similarity relationship based on granular-balls is as follows: 
\begin{definition}\label{def1}
	\setlength{\parindent}{2em}	Assume there is $\left\langle U,A,D \right\rangle$, converted into a set of granular-balls as ${U}'= \left\{{GB}_i,i=0,1,...,m\right\}$. The center and radii of ${GB}_i$ are denoted as $c_i$ and $r_i$ respectively. For any attribute a in $A$, a fuzzy binary relationship ${GBR}_a$ can be derived. If the following conditions are met, ${GBR}_a$ is said to be a fuzzy similarity relationship.	\\
	(1) Reflexivity: ${GBR}_a({GB}_i,{GB}_i)=1,\forall{GB}_i\in G$;	\\
	(2) Symmetry: ${GBR}_a({GB}_i,{GB}_j)={GBR}_a({GB}_j,{GB}_i)$, $\forall{GB}_i,{GB}_j\in G$.	\\
	The fuzzy similarity here is calculated through the general distance metric as follows:
	\begin{equation}\label{GBR}
	 {GBR}_a\left({GB}_i,{GB}_j\right) = 1- \frac{1}{\sqrt[p]{C}}\mathrm{\Delta}_p^a\left(c_i,c_j\right),
    \end{equation}
	where $c_i$ and $c_j$ are the centers of $GB_i$ and $GB_j$ respectively. $C$ is a fixed distance parameter that places the value of ${GBR}_a$ in the interval $[0.1]$. 
\end{definition}

In this article, we set $p=2$, then $\mathrm{\Delta}_2^a(c_i,c_j)$ represents the Euclidean distance between the centers $c_i$ and $c_j$ in the case of attribute $a$. If $B\subseteq A$, ${GBR}_B$ is a fuzzy similarity relationship under attribute subset $B$, expressed as ${GBR}_B{=(GBr_{ij}^B)}_{m\times m}$, where $0\le{GBr}_{ij}^B\le1,i,j=1,2,\ldots,m$. $B$ is an attribute subset. Here, $C$ is a fixed distance parameter: the number of attribute sets $B$.
	
This article sets the fuzzy similarity relationship to 1 only if the center points overlap and have the same radius. For ring-balls (a large ball envelops a small ball), two factors affect the fuzzy similarity relationship's membership degree: parameter $C$ and attribute subset $B$. For a given parameter $C$, as the number of attributes in $B$ increases, the degree of membership becomes smaller.

\begin{property}\label{1}
Let $B\subseteq A$, then ${GBR}_A\subseteq{GBR}_B$.
\end{property}
\begin{proof}
    Given a decision table $\left \langle {U}', A, D \right \rangle$ and sample space ${U}'= \left \{ GB_{i},i=1,2,...,m  \right \}$. Assume that the entire domain is divided into $l$ clear decision equivalence classes, denoted as ${U}'/D =  \left \{ D_{1},D_{2},...,D_{l} \right \}$. For $\forall GB_{i}, GB_{j} \in G$, we have $GB_{i} = \left \{ x_i,i=1,2,...,n_i \right \}, GB_{j} = \left \{ x_j,j=1,2,...,n_j \right \}$. Under the condition of attribute $B$, let $d = \left |B  \right |$, then $x_{i} = \left \{x_{i1},x_{i2},...,x_{id}  \right \}, x_{j} = \left \{x_{j1},x_{j2},...,x_{jd}  \right \}$. At this time, the centers of $GB_{i}$ and $GB_{j}$ containing attribute $B$ are
    \begin{align}
        c_{i} = \frac{1}{n_{i}}\sum_{i=1}^{n_{i}} x_{i} =\left ( c_{i1}, c_{i2},...,c_{id} \right ) , \\
        c_{j} = \frac{1}{n_{j}}\sum_{j=1}^{n_{j}} x_{j}=\left ( c_{j1}, c_{j2},...,c_{jd} \right ).
    \end{align}
    Similarly, assuming ${d}' = \left | A \right |$, because of $B\subseteq A$, so $d 
\le {d}'$. We have ${x_{i}}' = \left \{x_{i1},x_{i2},...,x_{id},...,x_{i{d}'} \right \}, {x_{j}}'  = \left \{x_{j1},x_{j2},...,x_{jd},...,x_{j{d}'}  \right \}$. Under attribute $A$, the centers of $GB_{i}$ and $GB_{j}$ are 
    \begin{align}
        {c_{i}'} = \frac{1}{n_{i}}\sum_{i=1}^{n_{i}} {x_{i}'} = \left ( c_{i1}, c_{i2},...,c_{id},...,{c_{i{d}'}'} \right ) ,\\   {c_{j}'} = \frac{1}{n_{j}}\sum_{j=1}^{n_{j}} {x_{j}'} = \left ( c_{j1}, c_{j2},...,c_{jd},...,{c_{j{d}'}'} \right ) .
    \end{align}
Therefore, the distances between the centers under the two attribute conditions are
\begin{equation*}
   \begin{aligned}
&\bigtriangleup ^{B} \left ( c_{i}, c_{j}\right ) \\ 
&= \sqrt{\left (  c_{i1}- c_{j1} \right )^{2} +  \left (  c_{i2}- c_{j2} \right )^{2}+...+\left (  c_{id}- c_{jd} \right )^{2}}; \\
& \bigtriangleup ^{A} \left ( {c_{i}}', {c_{j}}'\right ) \\
&= \sqrt{\left (c_{i1}- c_{j1} \right )^{2} +...+\left (c_{id}-c_{jd} \right )^{2}+...+\left (c_{i{d}'}- c_{j{d}'} \right )^{2}} .
   \end{aligned}
\end{equation*}
Therefore,
\begin{equation*}
\bigtriangleup ^{B} \left ( c_{i}, c_{j}\right ) \le \bigtriangleup ^{A} \left ( {c_{i}}', {c_{j}}'\right ).
\end{equation*}
That is, $r_{ij}^{B} \ge r_{ij}^{A}$, so $GBR_{A} \le GBR_{B}$.

\end{proof}

\begin{definition}\label{def2}
	\setlength{\parindent}{2em} Given a decision table $\left\langle {U}',A,D\right\rangle$ and ${U}'= \left \{ GB_{i},i=1,2,...,m  \right \}$. Assume that the whole domain is divided by $D$ into $l$ clear decision equivalence classes, expressed as $U'/D= \left \{ D_1,\ D_2,\ldots,D_l \right \}$. ${GBR}_B$ is the fuzzy similarity relationship on ${U}'$ derived from equation (\ref{GBR}) under the attribute subset $B$. For $\forall GB_i,GB_j\in U'$, the upper and lower approximations of decision class $D_j$ with respect to $B$ are defined as
\begin{equation}
\begin{aligned}
\underline{{GBR}_B}\left(D_j\right)\left(GB_j\right)&=\min_{GB_i\notin D_j} \left \{ 1 - GBR_B(GB_i, GB_j)\right \}, \\
&\ \ \ \ \ \ \ \ \ \ \ \ \ \ \ \ \ \ \ \ \ \ \ \ \ \ \ \ \ \ \  \ \ GB_j \in {U}' ,\\
\overline{{GBR}_B}(D_j)(GB_j)&=\max_{GB_i\in D_j} \left \{ GBR_B(GB_i, GB_j) \right \},  \\
&\ \ \ \ \ \ \ \ \ \ \ \ \ \ \ \ \ \ \ \ \ \ \ \ \ \ \ \ \ \ \  \ \ GB_j \in {U}',
\end{aligned}   
\end{equation}
where the lower approximation of $D_j$ is also called the positive domain of $D_j$.  
\end{definition}

Convert the upper and lower approximations of fuzzy similarity relationships into upper and lower approximations of distance, then the upper and lower approximations of decision class $D_j$ with respect to $B$ are defined as:
\begin{align}
\underline{{GBR}_B}(D_j)(GB_j)&=\frac{1}{\sqrt C}\min_{GB_i\notin D_j}{\mathrm{\Delta}^B(c_i,c_j)}, GB_j \in {U}' \label{low}\\
\overline{{GBR}_B}(D_j)(GB_j)&=\max_{GB_i\in D_j} \left \{1-\frac{1}{\sqrt C}\mathrm{\Delta}^B(c_i,c_j)\right \} , GB_j \in {U}' . 
\end{align} 

It can be easily seen that the lower approximate membership of object $GB_j$ to $D_j$ is based on the nearest samples belonging to other decision equivalence classes. The degree of membership is affected by the distance parameter $C$ and the attribute subset $B$ for a given distance parameter $C$, the greater the number of attributes in $B$, the greater the membership degree of downward approximation.

Based on the granular-ball fuzzy lower approximation, the granular-ball fuzzy positive domain of decision class $D$ concerning $B$ can be expressed as
\begin{equation}
GBPOS_B(D)=\bigcup_{i=1}^{l}{\underline{{GBR}_B}(D_i)}. 
\end{equation}


When the sample points are replaced by granular-balls, the sample space $U$ is transformed into $U'$. According to the definition of the fuzzy positive domain based on granular-balls, the granular-ball fuzzy dependence function can be calculated using the following formula:
\begin{equation}
{{W\partial}_B}^{'}(D)\mathrm{=} \frac{\sum_{i=1}^{m}{GBPOS_B(D)(GB_i)}}{\left|U'\right|} .
\end{equation}

Here, $\left | {U}'  \right |$ represents the number of granular-balls. However, the efficacy of the fuzzy dependency function hinges upon the quality of these granular-balls. To address this, we propose weighting the fuzzy dependency function based on the proportion of each granular-ball. Consequently, the weighted granular-ball fuzzy dependency function is defined as follows:
\begin{definition}\label{wgbfd}
Assuming that the original sample space is $U$, the space after the granular-balls are generated is ${U}'$, $GBPOS_B(D)(GB_i)$ is the granular-ball fuzzy positive domain, then the weighted granular-ball fuzzy dependency function is
\begin{equation}\label{GBdependency}
   \begin{aligned}
 {W\partial}_B(D) & \mathrm{=} \frac{\left|GB_1\right|}{\left|U\right|} \times GBPOS_B(D)(GB_1) + \\
&  \frac{\left|GB_2\right|}{\left|U\right|}\times GBPOS_B(D)(GB_2)+  \\ & ...+ \frac{\left|GB_m\right|}{\left|U\right|} \times GBPOS_B(D)(GB_m) \\ 
& \mathrm{=}\frac{\sum_{i=1}^{m}{ \{ \left|GB_i\right|\times GBPOS_B(D)(GB_i)} \} }{\left|U\right| } .
   \end{aligned}
\end{equation}
\end{definition}
The weighted fuzzy dependency function preserves the information from the original space sample points, ensuring a comprehensive representation. 

 \begin{theorem}\label{the}
When each granular-ball contains only one sample point, the weighted granular-ball fuzzy dependency function aligns with the traditional fuzzy dependency function.
 \end{theorem}
\begin{proof}
If each granular ball consists of only one sample point, then the number of granular-balls, 
$m$, is equal to the number of sample points in the dataset, $n$, and $ \left| GB i  \right| =1$. Consequently, we have 
\begin{equation}\label{GBdependency_1}
	\begin{aligned}
		{W\partial}_B(D) & \mathrm{=} \frac{\left|GB_1\right|}{\left|U\right|} \times GBPOS_B(D)(GB_1) + \\
		&  \frac{\left|GB_2\right|}{\left|U\right|}\times GBPOS_B(D)(GB_2)+  \\ & ...+ \frac{\left|GB_m\right|}{\left|U\right|} \times GBPOS_B(D)(GB_m) \\ 
		& \mathrm{=}\frac{\sum_{i=1}^{m}{ \{ \left|GB_i\right|\times GBPOS_B(D)(GB_i)} \} }{\left|U\right| } \\
		&  \mathrm{=} \frac{\sum_{i=1}^{n}{ \{  1 \times GBPOS_B(D)(GB_i)} \} }{\left|U\right| } \\
		&  \mathrm{=}\frac{\sum_{i=1}^{n}{ \{ POS_B(D)(x_i)} \} }{\left|U\right| } = \partial _{B} \left ( D \right ) .
	\end{aligned}
\end{equation}
Then the weighted granular sphere fuzzy dependency function is the traditional point-based fuzzy dependency function. This consistency reinforces the validity of our definition relative to the traditional approach.

\end{proof}

According to Property \ref{1} and Definition \ref{wgbfd}, we can get:

\begin{property}\label{2}
Let $\left\langle U', A, D \right\rangle$ be the decision table, $B_1, B_2\subseteq A$, if $B_1\subseteq B_2$, then $GBPOS_{B_1}(D)\subseteq GBPOS_{B_2}(D)$.
\end{property}
\begin{proof}
Given a decision table $\left\langle U',A,D\right\rangle$, the sample space $U' =\left\{{GB}_i,i=1,2,...,m\right\}$. Let the whole domain be divided into $l$ clear decision equivalence classes $ U'/D = \{D_1,\ D_2,\ldots,D_l \}$. For $\forall GB_{i}, GB_{j} \in G$, $c_{i},c_{j}$ are the centers of the two balls respectively. Since $B_{1}, B_{2}\subseteq A$, $B_{1}\subseteq B_{2}$, according to Property \ref{2}, we have $\bigtriangleup ^{B_1} \left ( c_{i}, c_{j}\right ) \le \bigtriangleup ^{B_2} \left ( c_{i}, c_{j}\right ) $. Then,
\begin{equation*}
\frac{1}{\sqrt{C} } \min_{GB_{i}\notin D_{j} } \left \{  \bigtriangleup ^{B_1}  \left ( c_i, c_j \right )  \right \} \le \frac{1}{\sqrt{C} } \min_{GB_{i}\notin D_{j} } \left \{  \bigtriangleup ^{B_2}  \left ( c_i, c_j \right )  \right \}.
\end{equation*}
Therefore,
\begin{equation*}
\begin{aligned}
 {\textstyle \bigcup_{i=1}^{l}} \underline{GBR_{B_1}} \left ( D_i \right ) &\subseteq {\textstyle \bigcup_{i=1}^{l}} \underline{GBR_{B_2}} \left ( D_i \right ) \\
 GBPOS_{B_1}\left ( D \right ) &\subseteq  GBPOS_{B_2}\left ( D \right ).
\end{aligned}
\end{equation*}
\end{proof}

Thus, we have the following theorem:

\begin{property}\label{3}
 Let $\left\langle U',A,D\right\rangle$ be the decision table, $B_1,B_2\subseteq A$, if $B_1\subseteq B_2$, then ${W\partial}_{B_1}(D)\le{W\partial}_{B_2}(D)$.
\end{property}
Property \ref{3} shows that the granular-ball fuzzy dependency is monotonically increasing as the size of the attribute subset changes.

\subsection{Feature selection of the Granular-ball Fuzzy Rough Set}
The lower approximation of fuzzy rough sets is often employed for feature selection, effectively eliminating irrelevant or redundant features, thereby mitigating noise and interference, enhancing model performance, and bolstering generalization ability. This is primarily attributed to the ability of the lower approximation of fuzzy rough sets to furnish valuable insights into feature relationships. We can pinpoint the most pertinent and impactful feature subsets by computing the lower approximation and subsequently leveraging fuzzy dependencies to assess the significance of each feature in the target variable. This approach facilitates a deeper understanding of underlying data patterns and regularities while augmenting the accuracy and efficiency of feature selection. In comparison with conventional methods, the lower approximation of fuzzy rough sets affords a more holistic consideration of feature relationships, thereby enabling more effective feature selection. The utility of an attribute or attribute subset hinges upon its relevance and redundancy. Therefore, we present the following definition:

\begin{definition}\label{def7}
	\setlength{\parindent}{2em} Let $\left\langle U',A,D \right\rangle$ be the decision table, $B\subseteq A$. For $\forall a\in B$, if ${W\partial}_{B-a}(D)\neq{W\partial}_B(D)$, attribute $a$ is indispensable in $B$. Otherwise, $a$ is redundant in $B$.
\end{definition}

After the above definition, we introduce the next definition, which will further clarify what a simplification of an attribute set is. In this definition, we will evaluate whether the attribute subset $B$ is a simplification of the attribute set $A$, that is, whether it has the same fuzzy rough dependency as $A$.

\begin{definition}\label{def4}
	\setlength{\parindent}{2em} Let $\left\langle U', A, D \right\rangle$ be the decision table, $B\subseteq A$. If the following definition is satisfied, then $B$ is a reduction of $A$.
\begin{align}
	  {W\partial}_{B-a}\left(D\right)& <{W\partial}_B\left(D\right), \forall a\in B, \label{13}   \\
	 {W\partial}_B(D)&={W\partial}_A(D).
\end{align}
\end{definition}
This definition states that a reduction is a minimal subset of conditional attributes with the same fuzzy rough dependencies as the entire set of attributes. According to Definition \ref{def7} and \ref{def4}, we prove the monotonicity of the dependency function using an increasing sequence of numbers as the asymptotic value of the parameter $C$.
\begin{theorem}\label{the1}
    Let $\left\langle U', A, D\right\rangle$ be the decision table, and $B\subseteq A$. A single increment sequence $\left \{ C_i:i=1,2,\dots,q \right \} $ is expressed as the range of values of the distance parameter $C$. When $C=C_i$, the fuzzy dependence of $D$ on $B$ is denoted by ${W\partial}^{i}_B(D)$. When $C=C_{i+1}$, the fuzzy dependence of $D$ on $B$ is denoted by ${W\partial}_B^{i+1}(D)$, then 
\begin{equation} 
	{W\partial}_{B}^{i+1}(D)=\sqrt{\frac{C_{i}}{C_{i+1}}}{W\partial}_{B}^{i}(D).
\end{equation}
\end{theorem}

\begin{proof}
Suppose the sample space $U' = \left \{GB_{j} , j = 1,2,..,m \right \} $ and $U'/D=\left \{D_1,D_2,\dots,D_l\right \} , D_j \in U'/D$. For any $GB_i,GB_j \in U' $, the centers are $c_i$ and $c_j$. according to formula (\ref{GBR}), we have
\begin{equation*} 
GBR_{ij}^B=1-\frac{1}{\sqrt C}\mathrm{\Delta} ^B(c_i,c_j).
\end{equation*}
From the definition of granular-ball fuzzy approximation, it follows that
\begin{equation*} 
\underline{{GBR}_B}(D_j)(GB_j)=\min_{GB_i\notin D_j}\left \{\frac{1}{\sqrt C}\mathrm{\Delta} ^B(c_i,c_j)\right \} , GB_j\in U.
\end{equation*}
The fuzzy positive domain of the granular-ball $GB_j$ with attribute set $B$ can be expressed as:
\begin{equation*} 
\begin{aligned}
GBPOS_B&(D)(GB_j) =\bigcup_{t=1}^{l} \frac{1}{\sqrt C}\left \{\min_{GB_i\notin D_j}{\mathrm{\Delta}^B(c_i,c_j)}\right \}   \\
&=\max_{t=1}^{l}  \frac{1}{\sqrt C}\left \{\min_{GB_i\notin D_j}{\mathrm{\Delta}^B(c_i,c_j)}\right \} , GB_j \in U' .
\end{aligned}
\end{equation*}
Denote
\begin{equation*} 
GBPOS_B^i(D)(GB_j)=\max_{t=1}^{l}  \frac{1}{\sqrt C_i}\left \{\min_{GB_i\notin D_j}{\mathrm{\Delta}^B(c_i,c_j)}\right \} ,
\end{equation*}
and
\begin{equation*} 
GBPOS_B^{i+1}(D)(GB_j)=\max_{t=1}^{l}  \frac{1}{\sqrt C_{i+1}}\left \{\min_{GB_i\notin D_j}{\mathrm{\Delta}^B(c_i,c_j)} \right \} .
\end{equation*}

Then,

\begin{equation*} 
\begin{aligned}
&{W\partial}_B^i(D)=\frac{\sum_{i=1}^{m}\left(\left|GB_i\right|\times GBPOS_B^i(D)(GB_i)\right)}{\left|U\right|}\\ &=\frac{\sum_{i=1}^{m}\left(\left|GB_i\right|\times\frac{1}{\sqrt{C_i}}\max_{t=1}^{l} \left \{ \min_{GB_i\notin D_j}{\mathrm{\Delta}^B(c_i,c_j)} \right \} \right)}{\left|U\right|}\\&=\frac{\frac{1}{\sqrt{C_i}}\sum_{i=1}^{m}\left(\left|GB_i\right|\times \max_{t=1}^{l} \left \{\min_{GB_i\notin D_j}{\mathrm{\Delta}^B(c_i,c_j)}\right \} \right)}{\left|U\right|},
\end{aligned}
\end{equation*}

\begin{equation*}
\begin{aligned}
&{W\partial}_B^{i+1}(D)\mathrm{=}\frac{\sum_{i=1}^{m}\left(\left|GB_i\right|\times POS_B^{i+1}(D)(GB_i)\right)}{\left|U\right|}\\
&\mathrm{=}\frac{\sum_{i=1}^{m}\left(\left|GB_i\right|\times\frac{1}{\sqrt{C_{i+1}}}\max_{t=1}^{l} \left \{\min_{GB_i\notin D_j}{\mathrm{\Delta}^B(c_i,c_j)}\right \} \right)}{\left|U\right|}\\
&\mathrm{=}\frac{\frac{1}{\sqrt{C_{i+1}}}\sum_{i=1}^{m}\left(\left|GB_i\right|\times \max_{t=1}^{l} \left \{ \min_{GB_i\notin D_j}{\mathrm{\Delta}^B(c_i,c_j)} \right \}  \right)}{\left|U\right|}.
\end{aligned}
\end{equation*}
Thus,
\begin{equation*}
{W\partial}_B^{i+1}(D)=\sqrt{\frac{C_i}{C_{i+1}}}{W\partial}_B^i(D).
\end{equation*}

\end{proof}
    

For the importance of an attribute to the attribute set, we have the following definition:

\begin{definition}\label{def5}
	\setlength{\parindent}{2em} Let $\left\langle U',A,D\right\rangle$ be the decision table, $B\subseteq A$, $a\in A-B$, the significance of a relative to $B$ is defined as:
\begin{equation} 
	SIG\left(a,B,D\right)={W\partial}_{B\cup a}\left(D\right)-{W\partial}_B\left(D\right).
\end{equation}
\end{definition}

For a given subset of attributes $B$, we can get different habit dependencies for different distance parameters $\{C_i:i=1,2,\dots,q\}$. When we use incremental sequences to calculate fuzzy dependencies, Theorem \ref{the1} gives the relationship between fuzzy dependencies. With Theorem \ref{the1} we can easily get an iterative formula to calculate the importance of an attribute.

 \begin{theorem}\label{the2}
Let $\left\langle U',A,D\right\rangle$ be the decision table, $B\subseteq A$, and the natural sequence $N=\{1,2,3\dots\}$ for the range of values of the parameter $C$. Let $C=i=\left|B\right|$, and the importance of $a$ relative to $B$ is denoted by $SIG(a,B,D)$, then for any $a\in A-B$, it satisfies:
\begin{equation}\label{sigg}
\begin{aligned}
SIG^i\left(a,B,D\right)&={W\partial}_{B\cup  \{a\}}^{i+1}\left(D\right)-{W\partial}_B^{i+1}\left(D\right)    \\
& =   {W\partial}_{B\cup  \{a\}}^{i+1}\left(D\right)-\sqrt{\frac{i}{i+1  }}{W\partial}_B^i\left(D\right) \\
& = {W\partial}_{B\cup \{a\}}^{i+1}\left(D\right)-\sqrt{\frac{\left|B\right|}{\left|B \bigcup  \{a\}  \right|  }}{W\partial}_B^i\left(D\right).
\end{aligned}
 \end{equation}
\end{theorem}
According to Theorem \ref{the2}, the significance of a candidate attribute can be computed by formula (\ref{sigg}).

According to equation (\ref{13}) in Definition \ref{def4}, ${W\partial}_{B\cup a}^{i+1} > {W\partial}_B^i$, and $\sqrt{\frac{\left|B\right|}{\left|B \bigcup a  \right|  }} < 1$. Then for any $a \in A-B$, we have $SIG^i\left(a,B,D\right)>0$. Next, the monotonicity of the dependence is proved when superimposing the attributes.


\begin{theorem}
Let $\left\langle U',A,D\right\rangle$ be the decision table, $B\subseteq A$. The natural sequence of the value range of parameter $C$ is $N=\{1,2,3\dots\}$. Let $C=i=\left|B\right|$, then $SIG^i\left(a,B,D\right)$ converges to 0 when $i\rightarrow\left|A\right|$.
\end{theorem}
\begin{proof}
Based on the importance of attributes, we assume that the order of attributes is determined as follows:
\begin{equation*}
\begin{aligned}
a_{j1}  \succ  a_{j2} \succ  a_{j3} \succ  \dots \succ  a_{j|A|-1} \succ  a_{j|A|}.
\end{aligned}
 \end{equation*}
Since $a_{jk} \succ a_{jk+1}$, it means that the attribute $a_{jk}$ is more important than $a_{jk+1}$. For any $B\subseteq A$, let $a_{jk},a_{jk+1}\notin B$. It follows the classical fuzzy rough set theory as follows:
\begin{equation*}
\begin{aligned}
{W\partial}_B^{C_i}(D)\le{W\partial}_{B\cup \{ {a_{jk}} \}}^{C_i}(D)\le{W\partial}_{B\cup\{ {a_{jk}} \}\cup     \{ {a_{jk+1}} \}  }^{C_i}(D),
\end{aligned}
 \end{equation*}
\begin{equation*}
\begin{aligned}
&{W\partial}_{B\cup  \{ {a_{jk}} \}}^{C_i}(D)-{W\partial}_B^{C_i}(D) \geq  \\ & \ \ \ \ \ \ \ \ \  \ {W\partial}_{B\cup \{ {a_{jk}} \}   \cup   \{ {a_{jk+1}}\} }^{C_i}(D)-{W\partial}_{B\cup  \{ {a_{jk}} \} }^{C_i}(D),
\end{aligned}
 \end{equation*}
where, ${W\partial}_B^{c_i}(D)$ denotes the fuzzy dependency relationship when ${C=C_i}$. According to Theorem \ref{the1}, we have:
\begin{equation*}
\begin{aligned}
{W\partial}_B^{i+1}\left(D\right)=\sqrt{\frac{C_i}{ {i+1}}}{W\partial}_B^{C_i}\left(D\right),
\end{aligned}
 \end{equation*}
\begin{equation*}
\begin{aligned}
{W\partial}_B^i(D)=\sqrt{\frac{C_i}{i}}{W\partial}_B^{C_i}(D),
\end{aligned}
 \end{equation*}
and
\begin{equation*}
\begin{aligned}
{W\partial}_{B\cup  \{ {a_{jk}} \} \cup   \{ {a_{jk+1}} \} }^{i+2}=\sqrt{\frac{C_i}{i+2}}{W\partial}_{B\cup \{ {a_{jk}} \}   \cup  \{ {a_{jk+1}} \} }^{C_i}(D).
\end{aligned}
 \end{equation*}
According to Theorem \ref{the2}, we have
\begin{equation*}
\begin{aligned}
&SIG^i\left(a_{jk},B,D\right)\\  & ={W\partial}_{B\cup   \{{a_{jk}} \} }^{i+1}\left(D\right)-\sqrt{\frac{i}{i+1}}{W\partial}_B^i\left(D\right) ,  \\
&=\sqrt{\frac{C_i}{i+1}}\left({W\partial}_{B\cup \{{a_{jk}} \}  }^{i+1}\left(D\right)-{W\partial}_B^{C_i}\left(D\right)\right)，
\end{aligned}
 \end{equation*}
and
\begin{equation*}
\begin{aligned}
&SIG^{i+1}\left(a_{jk+1}, B \cup \{a_{jk}\},D\right) \\&={W\partial}_{B\cup\{{a_{jk}} \}\cup\{{a_{jk+1}} \}}^{i+2}\left(D\right)-\sqrt{\frac{i+1}{i+2}}{W\partial}_{B\cup\{{a_{jk}} \}}^{i+1}\left(D\right)    \\
&=\sqrt{\frac{C_i}{i+2}}\left({W\partial}_{B\cup\{{a_{jk}} \}\cup\{{a_{jk+1}} \}}^{C_i}\left(D\right)-{W\partial}_{B\cup\{{a_{jk}} \}}^{C_i}\left(D\right)\right) .
\end{aligned}
 \end{equation*}
Therefore, 
\begin{equation*}
\begin{aligned}
SIG^i\left(a_{jk},B,D\right)\geq SIG^{i+1}\left(a_{jk+1}, B \cup \{a_{jk}\},D\right).
\end{aligned}
 \end{equation*}
$SIG^i\left(a,B,D\right)$ monotonically decreases as $i$ increases. Combining for any $a\in \{ A-B\},\ SIG^i\left(a,B,D\right)>0$, we know $SIG^i\left(a,B,D\right)$ is bounded monotonic. It can be concluded that $SIG^i\left(a,B,D\right)$ converges to 0 when $i\rightarrow\left|A\right|$.
\end{proof}
This theorem states that the importance of candidate attributes diminishes as the number of selected attributes approaches the cardinality of the attribute set $A$. In essence, formula (\ref{sigg}) maintains the convergence of attribute importance, indicating that the monotonicity of the dependency function can be leveraged to design a heuristic attribute reduction algorithm.        
         
     \begin{algorithm}[h!]  
     	\caption{Feature selection based on GBFRS}  
     	\label{alg}
     	\textbf{Input}: A dataset $ D=\{x_{1}, x_{2}, ..., x_{n}\}$ with the attribute set $B$, the Purity threshold $T$;  \\
     	\textbf{Output}: Attribute set $B^{''}$; 
     	\begin{algorithmic}[1]  
     		\STATE The data set $D$ with the attribute set $B$ is split into granular-balls by k-means algorithm to generate $\sqrt{n}$ granular-balls, which are recorded as a set $G= \left \{GB_{1}, GB_{2}, ..., GB_{\sqrt{n}} \right \} $. Set $G'=\emptyset$.
     		\REPEAT 
     		\FOR{each granular-ball $GB_{i} \in G$}
     		\IF {$purity(GB_{i})<T$ }
     		\STATE Split $GB_{i}$ into $2$ sub-granular-balls $\{GB_{j}^{'},j=1,2\}$ 
     		\STATE $G = G - \{ GB_{i} \}$, $G =  G +  \{GB_{j}^{'}\}$
     		\ELSE 
     		\STATE $G'$ = $G'$ +  $ \{ GB_{i} \}$, $G = G - \{ GB_{i} \}$
     		\ENDIF
     		\ENDFOR
     		\UNTIL {$G=\emptyset$}
     		\STATE Remove the overlap between heterogeneous granular-balls and obtain the final granular-ball set $U'$.
     		\STATE The dimension of each sample point is $d$, and $a_i$ is the index corresponding to the $i$-th attribute;
     		\STATE ${W\partial}=0$, final attribute set $B^{''} = \emptyset$ and current attribute set $B^{'}=\{a_1,a_2,...,a_d\}$;
     		\WHILE{$ B^{'}=\emptyset$ or $max\_W\partial \leq {W\partial}$}
     		\FOR{$a_i \in B^{'}$}
     		\STATE $B^{''} = B^{''} +\left \{a_i \right \} $;
     		\STATE In the case of the current attribute set $B^{''}$, the fuzzy lower approximation of each granular-ball is solved by the formula (\ref{low}), and then the weighted fuzzy dependency ${W\partial}_B^{''}$ is solved by the formula (\ref{GBdependency});
     		\IF { $ {W\partial} > {W\partial}_B^{''}$}  
     		\STATE $ {W\partial}= {W\partial}_B^{''}, N_i = a_i$ // Select the attribute set $B^{'}$, the single attribute with the greatest dependence
     		\ENDIF
     		\ENDFOR
     		\STATE Select the maximum dependency and its corresponding attribute $N_i$;
     		\IF {$ max\_W\partial < {W\partial}$}
     		\STATE $max\_W\partial={W\partial}$, $B^{''} = B^{''} + \{N_i\}$, $B^{'}= B^{'} - \{N_i\}$; // The set $B^{''}$ is the final feature selection result;
     		\ENDIF
     		\ENDWHILE
     	\end{algorithmic}  
     \end{algorithm}

\subsection{Algorithm Design \label{sec:algori}}
This section discusses the algorithm design of GBFRS for feature selection based on weighted granular-ball dependency. The algorithm is divided into two parts.

The first part is for granular-balls generation. Given a data set $D= \left \{ x_{1}, x_{2}, ..., x_{n}\right \}$ with attribute set $B$, and a purity threshold given as $T$. First, the data set is divided into $\sqrt{n}$ granular-balls using the k-means algorithm. Then, it is determined whether the purity of each ball meets the purity threshold. The granular-balls that do not meet the purity threshold need to be classified using the k-means algorithm. Until all balls meet the purity threshold. Repeat the above process until the purity of all granular-balls is not less than the threshold $T$. Finally, remove the overlapping parts between heterogeneous balls to obtain the final granular-ball set $U′$.

The second part is to select attributes according to the weighted dependency function. In the process of selecting the attribute set, the fuzzy lower approximation of each ball is calculated for the obtained granular-ball set $U'$, and the weighted sphere fuzzy dependency ${W\partial}_{B}$ is calculated by the formula (\ref{GBdependency}). For each loop, a single attribute needs to be superimposed, and the attribute $N_i$ with the largest weighted dependency is selected, and then the attribute set is updated. If the dependency increases after adding $N_i$, it is added to the final attribute set. Repeat this process until the dependency no longer increases and finally obtain the optimal attribute set. The detailed algorithm process is shown in Algorithm 1.

\section{Experiment}\label{sec:experiment}

To demonstrate the feasibility and effectiveness of GBFRS, we compared it with two popular algorithms, FNRS \cite{wang2016feature}, HANDI \cite{wang2017feature}, FAR\_FIE \cite{hu2006information}, FS\_NDEM \cite{hu2009selecting} and FRDMAR \cite{degang2010local}  algorithm. As shown in Table \ref{tab:1}, we randomly select 9 UCI data sets and introduce their sample size, numbers of attributes and categories. The GBFRS method based on pellet ball calculation involves the optimization of the purity threshold of each ball, which we optimize in the range of [0.6,1] with a step size of 0.05. To ensure the quality of the reduced attribute set is classifier-independent and to prevent overfitting, we validated it using the k-nearest neighbors (kNN) algorithm. All methods underwent a 5-fold cross-validation experiment. Computational experiments were conducted on a PC with an Intel Core i7-10700 CPU @ 2.90GHz and 32 GB RAM, using Python 3.9.


\begin{table}[!h]
\centering
\caption{DATASET INFORMATION}
\label{tab:1}
\begin{tabular}{ccc}
\hline
Dataset                               & Samples & Attributes  \\ \hline
zoo                                   & 101     & 16      \\
lymphography                          & 148     & 18      \\
primary-tumor                         & 336     & 15      \\
backup-large                          & 625     & 4       \\
iono                                  & 351     & 34      \\
Diabetes                              & 768     & 8       \\
wdbc                                  & 569     & 30      \\
audit\_risk                           & 776     & 17      \\
wine                                  & 178     & 13      \\
Parkinson\_Multiple & 1040    & 27    \\                             \hline
\end{tabular}
\end{table}

\begin{table*}[!h]
\setlength{\tabcolsep}{3pt}
\caption{Comparison of accuracy between different methods with different levels of class noise}
\label{tab:2}
\renewcommand\arraystretch{1.3}
\fontsize{6}{6}\selectfont
\setlength{\tabcolsep}{0.4mm}{ 
\begin{tabular*}{\linewidth}{@{}cccccccccccccc@{}}
\hline
Dataset      &       & zoo    & lymphography & primary-tumor     & backup-large & iono   & Diabetes & wdbc   & audit\_risk    & wine    & Parkinson\_Multiple     \\ \hline
\multirow{6}{*}{0\%}  
& FAR\_FIE  & 0.9500 $\pm $0  & 0.7871  $\pm $0.7888    & \textbf{0.7370  $\pm $ 0.6667 }        & 0.9758    $\pm $ 0.2285    &  0.8671   $\pm $ 0.7284  & 0.6854   $\pm $ 0.5897     & 0.8949   $\pm $ 0.7190   & 0.9331   $\pm $ 0.3189   & 0.7537   $\pm $  1.1670     & 0.4818    $\pm $ 1.2962                                                           \\
& FS\_NDEM & 0.6760 $\pm $0.5477  & 0.7170  $\pm $1.1383    & 0.6943  $\pm $ 0.4995         & 0.8549   $\pm $ 0.5937   &  0.7486  $\pm $ 0  & 0.7040  $\pm $ 0.4790    & 0.9370  $\pm $ 0.6274  & 0.7071  $\pm $ 1.0105    & 0.8678  $\pm $  0.5053    & 0.5261   $\pm $ 0.5637                                                          \\
& FRDMAR  & \textbf{0.9800 $\pm $0  } & 0.7361  $\pm $1.6241     & 0.7343   $\pm $ 0.9201          & 0.9758   $\pm $ 0.2923   &  0.8663  $\pm $ 0.4238   &0.7051  $\pm $ 0.5713  & 0.9000 $\pm $ 0.5340  & 0.9323  $\pm $ 0.2131    & 0.7740   $\pm $  1.9571   & 0.4828   $\pm $ 0.7983                   \\
 & FNRS  & 0.5680$\pm $0.4472 & 0.7361 $\pm $0.8870    & 0.6567  $\pm $ 0         & 0.9190  $\pm $ 0.5846  &\textbf{ 0.8880 $\pm $ 0.5111} & 0.5445 $\pm $ 0.7410   & 0.8211 $\pm $ 0.3198 & 0.8078 $\pm $ 0.4040   & 0.7209  $\pm $  1.8132   & 0.4974  $\pm $ 1.0705                                                         \\
& HANDI & \textbf{0.9800$\pm $ 0} & 0.8313 $\pm $ 1.4748      & 0.7206 $\pm $ 0.7189      & \textbf{0.9863 $\pm $  0.1461 }   & 0.8840 $\pm $ 0.3258  & 0.6978  $\pm $ 0.5328  & \textbf{0.9540  $\pm $ 0.4591} & 0.9510 $\pm $ 0.3230  &  0.9514 $\pm $ 0.3094   & 0.9767 $\pm $  0.2756                                                         \\
& GBFRS & 0.9500$\pm $0.1118 & \textbf{0.8356 $\pm $ 0.0510 }    & 0.7116$\pm $ 0.0251        & 0.9768  $\pm $ 0.0189   & 0.7953 $\pm $ 0.1328 &  \textbf{0.7350 $\pm $ 0.0160 }& 0.9137 $\pm $ 0.0923 & \textbf{0.9908 $\pm $ 0.0136} & \textbf{0.9600 $\pm $ 0.0433}     & \textbf{1.0000 $\pm $ 0  }                              \\
                & & & & & & & & & & &\\
\multirow{6}{*}{5\%} 
& FAR\_FIE & 0.8870 $\pm $0.8233  & 0.6435  $\pm $0.5110    & \textbf{0.7107  $\pm $ 1.1120 }        & 0.8987    $\pm $ 0.4076    &  0.8266   $\pm $ 0.4675  & 0.6310   $\pm $ 0.7194     & 0. 7977  $\pm $ 0.6611   & 0.8087   $\pm $ 0.7215   & 0.7102   $\pm $  2.0810     & 0.4886    $\pm $ 0.5369                                                           \\
& FS\_NDEM & 0.6360  $\pm $0.8944   & 0.6884   $\pm $0.8870     & 0.6913   $\pm $ 0.6872           & 0.7889    $\pm $ 0.5937    &  0.7314   $\pm $ 0  & 0.6389   $\pm $ 0.6253     & 0.8644   $\pm $ 0.4489   & 0.6747   $\pm $ 0.6948     & 0.6836    $\pm $  1.1985     & 0.5128    $\pm $ 0.8318                                                           \\
& FRDMAR   & 0.9120  $\pm $0.4472   & 0.7007   $\pm $0.9621     & 0.7069   $\pm $ 0.5742           & 0.9039    $\pm $ 0.7523    &  0.8240   $\pm $ 0.5570  & 0.6480   $\pm $ 0.4421     & 0.7996   $\pm $ 0.3387   & 0.8174   $\pm $ 0.5220     & 0.6994   $\pm $  1.2883    & 0.4909    $\pm $ 0.3910                                                           \\
 & FNRS  & 0.5600$\pm $ 0 & 0.7483 $\pm $ 0   & 0.6418 $\pm $ 0       & 0.8869$\pm $  0.9694  & 0.8440$\pm $ 0.9815 & 0.5327 $\pm $  0.8069 & 0.7877 $\pm $ 1.0236 & 0.7331 $\pm $ 1.2787  & 0.6497 $\pm $  2.2949  & 0.5013  $\pm $ 0.4272                              \\
& HANDI & \textbf{0.9100 $\pm $0.7071} & 0.7810 $\pm $ 0.5692   & 0.7081 $\pm $  1.0638     & 0.9084 $\pm $ 0.0591     & 0.8360 $\pm $ 0.2556 & 0.6553 $\pm $ 0.3272 & 0.8592 $\pm $  0.4658 & 0.9118 $\pm $ 0.0917 & 0.8689 $\pm $  0.9283   &  0.8468  $\pm $  0.7529                                          \\
 & GBFRS & 0.8500$\pm $ 0.1708 & \textbf{0.8218  $\pm $ 0.0619 }   & 0.6810  $\pm $ 0.0503         & \textbf{0.9370$\pm $ 0.0321 }  & \textbf{0.8501$\pm $ 0.0826} & \textbf{0.6762 $\pm $ 0.0407} & \textbf{0.8767 $\pm $ 0.0647 } & \textbf{0.9486 $\pm $  0.0204}  & \textbf{ 0.9149 $\pm $ 0.0753 } &  \textbf{0.9518 $\pm $ 0.0091 }                                \\
                 & & & & & & & & & & &\\
\multirow{6}{*}{10\%} 
& FAR\_FIE & 0.8010 $\pm $0.9944  & 0.6667  $\pm $1.2827    & 0.6525  $\pm $ 1.2039       & 0.8023    $\pm $ 0.7893    &  0.7574   $\pm $ 0.4140  & 0.6001   $\pm $ 0.7169     & 0.7327   $\pm $ 0.7273   & 0.7298   $\pm $ 0.6892   & 0.7000   $\pm $  0.9397     & 0.4704    $\pm $ 0.8597                                                           \\
& FS\_NDEM & 0.6260  $\pm $0.5477   & 0.6367   $\pm $2.0180     & 0.5570   $\pm $ 1.1676           & 0.7418    $\pm $ 0.3268    &  0.7029   $\pm $ 0  & 0.5823   $\pm $ 0.5728     & 0.7761   $\pm $ 0.5892   & 0.6625   $\pm $ 1.0593     & 0.6486   $\pm $  1.6664    & 0.5047    $\pm $ 0.4936                                                           \\
& FRDMAR   & 0.7980  $\pm $3.0332   & 0.6544   $\pm $1.5513     & 0.6615   $\pm $ 0.5423        & 0.7908    $\pm $ 1.0589    &  0.7451   $\pm $ 0.4238  & 0.6042   $\pm $ 0.5874     & 0.7313   $\pm $ 0.5918   & 0.7362   $\pm $ 0.4990     & 0.6994   $\pm $  1.0867    & 0.4743    $\pm $ 1.1598                                                           \\
& FNRS  & 0.5160 $\pm $ 0.5477 & 0.6422 $\pm $ 1.0317    & 0.6125  $\pm $0.8805      & 0.7745 $\pm $ 0.6932  & \textbf{0.7663 $\pm $ 0.5111} & 0.5304 $\pm $ 0.8058 & 0.6944 $\pm $ 1.1383 & 0.7185  $\pm $ 0.3669  &  0.6475 $\pm $ 0.5053  & 0.5047 $\pm $ 0.5298                                                \\
& HANDI & 0.8120 $\pm $ 0.8367 & 0.6830 $\pm $ 0.7756  & 0.6353 $\pm $ 0.7127         & 0.7908 $\pm $ 0.4002  & 0.7571 $\pm $ 0.9689 & 0.6039 $\pm $ 0.6414 & 0.7873  $\pm $ 0.7714 & 0.7896  $\pm $ 0.6042  &  0.8373 $\pm $ 0.7366  &  0.8114  $\pm $ 0.2636                             \\
& GBFRS & \textbf{0.8256 $\pm $ 0.1724} & \textbf{0.7402  $\pm $ 0.1086} & \textbf{0.6738 $\pm $  0.02672 }    & \textbf{ 0.8837 $\pm $ 0.0425 } & 0.7492 $\pm $ 0.1391 & \textbf{0.6462 $\pm $ 0.0380} & \textbf{ 0.7938 $\pm $  0.0845} & \textbf{0.8946 $\pm $  0.0202 } & \textbf{0.8635  $\pm $  0.0845 }  & \textbf{ 0.8864  $\pm $  0.0401 }                            \\
                 & & & & & & & & & & &\\
\multirow{6}{*}{15\%}
& FAR\_FIE & 0.7550  $\pm $0.7071  & 0.6354  $\pm $1.3681    & 0.6039  $\pm $ 1.6712         & 0.7229    $\pm $ 0.6680    &  0.6686   $\pm $ 1.1349  & 0.6025   $\pm $ 1.0014     & 0.6937   $\pm $ 0.9092   & 0.7064   $\pm $ 0.6360   & 0.7203   $\pm $  0.8100     & 0.5247    $\pm $ 0.6284                                                           \\
& FS\_NDEM & 0.5760  $\pm $0.5477   & 0.6218   $\pm $0.6085     & 0.5093   $\pm $ 1.4096         & 0.6458    $\pm $ 0.7870    &  0.6543   $\pm $ 0  & 0.5802   $\pm $ 0.5608     & 0.7102   $\pm $ 0.4757   & 0.5977   $\pm $ 0.7009     & 0.6644   $\pm $  1.1716    & 0.5034    $\pm $ 0.6528                                                           \\
& FRDMAR   & 0.7540  $\pm $0.5477   & 0.6422   $\pm $1.6383     & 0.5958   $\pm $ 1.3114           & 0.7340    $\pm $ 0.5468    &  0.6760   $\pm $ 1.4084  & 0.6083   $\pm $ 0.4665     & 0.6926   $\pm $ 0.6540   & 0.7045   $\pm $ 0.7303     & 0.7175   $\pm $  1.1300    & 0.5253    $\pm $ 0.6026                                                           \\
 & FNRS  & 0.5020$\pm $ 0.4472 & 0.5946  $\pm $ 0.3726  & 0.4997 $\pm $ 0.7784       & 0.7242 $\pm $0.5468  & 0.6989 $\pm $ 0.5190 & 0.5648 $\pm $ 0.5638 & 0.6658 $\pm $ 0.9755  & 0.6804 $\pm $ 0.7256   & 0.5514 $\pm $ 0.8568  & 0.4945$\pm $ 1.0961                                                       \\
& HANDI & \textbf{0.8260 $\pm $ 1.1402} & 0.6653 $\pm $ 1.1178  & 0.6012 $\pm $  1.8594          & 0.7411 $\pm $ 1.1178   & 0.6937 $\pm $ 0.2390 & 0.6070 $\pm $ 0.9108 & 0.7211 $\pm $ 0.6886 & 0.7468  $\pm $ 0.4712  & 0.6712 $\pm $  1.0867  &  0.7014  $\pm $ 0.6064                                 \\
& GBFRS & 0.7165$\pm $0.0831  & \textbf{0.6697 $\pm $ 0.0664} & \textbf{0.6254 $\pm $ 0.0638 }        & \textbf{0.8348 $\pm $ 0.0492 }  & \textbf{0.7047 $\pm $ 0.0874 } & \textbf{0.6135  $\pm $  0.0285 } & \textbf{0.7639 $\pm $ 0.0972}  &\textbf{ 0.7845 $\pm $ 0.0206 } & \textbf{0.7565 $\pm $  0.1080 }  &  \textbf{0.7959 $\pm $ 0.0691}                                 \\
                 & & & & & & & & & & &\\
\multirow{6}{*}{20\%}
& FAR\_FIE & 0.6420  $\pm $1.1353  & 0.5116  $\pm $1.9454    & 0.5734  $\pm $ 0.9168       & 0.6425    $\pm $ 1.0125    &  0.6460   $\pm $ 0.8566  & 0.5945   $\pm $ 0.6053     & 0.6505   $\pm $ 0.8014   & 0.6621   $\pm $ 0.4939   & 0.5576   $\pm $  1.7270     & 0.5235    $\pm $ 0. 9045                                                          \\
& FS\_NDEM & 0.4160  $\pm $0.5477   & 0.5388   $\pm $2.3269     & 0.5618   $\pm $ 1.3744          & 0.5359    $\pm $ 0.7664    &  0.6377   $\pm $ 0.6515  & 0.6000   $\pm $ 0.4646     & 0.6377   $\pm $ 0.4015   & 0.5922   $\pm $ 0.7021     & 0.5876  $\pm $  1.3839    & 0.5396    $\pm $ 0.5554                                                           \\
& FRDMAR   & 0.6600  $\pm $1.5811   & 0.4993   $\pm $1.0317     & 0.5809   $\pm $ 1.0680        & 0.6444    $\pm $ 1.2741    &  0.6497   $\pm $ 0.5190  & 0.5849   $\pm $ 0.6936     & 0.6518   $\pm $ 0.4724   & 0.6488   $\pm $ 1.0513     & 0.5582   $\pm $  1.0107    & 0.5290    $\pm $ 0.2600                                                           \\
 & FNRS  & 0.5080 $\pm $0.4472 & 0.4748 $\pm $ 1.3941  & 0.5940 $\pm $ 7.9441        & 0.6471  $\pm $ 0.9243  & 0.6314$\pm $ 0.9476 & 0.5072 $\pm $ 1.2369 & 0.6144 $\pm $ 0.5426  & 0.4615 $\pm $ 0.5757  & 0.5390 $\pm $ 0.6442    & 0.5030 $\pm $  0.5678                        \\
& HANDI & \textbf{0.7400$\pm $0.7071} & 0.5184 $\pm $ 1.8877 & 0.5791 $\pm $ 1.5367       & 0.6497 $\pm $ 0.4261  & 0.6577 $\pm $ 0.7396 & 0.5909 $\pm $  1.0447 & 0.6585  $\pm $ 0.6587  & 0.6809 $\pm $ 0.1834  &  0.6316 $\pm $ 2.3088      & 0.6289 $\pm $ 0.6758                                      \\
& GBFRS & 0.6857$\pm $ 0.1793 & \textbf{0.6377 $\pm $ 0.0848 } & \textbf{0.6064 $\pm $ 0.0744}             &\textbf{ 0.7690 $\pm $ 0.0496 }  & \textbf{0.6579 $\pm $ 0.1347} & \textbf{0.6214  $\pm $ 0.0184 }& \textbf{0.6703 $\pm $ 0.0777}  & \textbf{0.7547  $\pm $ 0.0274 }  & \textbf{ 0.7681  $\pm $  0.1154  } &  \textbf{0.7505 $\pm $ 0.0180 }                \\
                 & & & & & & & & & & &\\
\multirow{6}{*}{25\%} 
& FAR\_FIE & 0.4880  $\pm $1.5492  & 0.4599  $\pm $1.8478    & 0.5758  $\pm $ 1.4381        & 0.6003    $\pm $ 1.0225    &  0.5849   $\pm $ 0.8417  & 0.5563   $\pm $ 0.8015     & 0.5835   $\pm $ 0.9137   & 0.5783   $\pm $ 0.5381   & 0.5271   $\pm $  1.5078     & 0.5090    $\pm $ 0.6721                                                           \\
& FS\_NDEM & 0.3020  $\pm $2.2804   & 0.4109   $\pm $1.8380     & 0.5427   $\pm $ 0.9990         &  0.5399   $\pm $ 1.3355    & \textbf{ 0.6057   $\pm $ 0}  &0.5533   $\pm $ 0.9783     & 0.5799   $\pm $ 0.8407   & 0.5261   $\pm $  1.4028    & 0.4893    $\pm $  0.9454    & 0.4832    $\pm $ 0.4463                                                           \\
& FRDMAR   & 0.5080  $\pm $0.8367   & 0.4367   $\pm $0.7452     & 0.5612   $\pm $ 1.2125         & 0.5895    $\pm $ 0.5937    &  0.5720   $\pm $ 1.3613  & 0.5549   $\pm $ 0.7580     & 0.5775   $\pm $ 0.7042   & 0.5681   $\pm $ 0.3998     & 0.4554   $\pm $  2.0988    & 0.5026    $\pm $ 0.8263                                                           \\
 & FNRS  & 0.4640$\pm $0.8944& 0.5211 $\pm $ 1.1383 & 0.5791 $\pm $ 0        & 0.6111 $\pm $ 0.5167  & 0.5806 $\pm $ 0.9125 & 0.5035 $\pm $ 0.5860 & 0.5898 $\pm $ 0.4489  & 0.5990 $\pm$   0.5139  & 0.4294 $\pm $ 1.2633  & 0.5080  $\pm $  0.7215                      \\
& HANDI & 0.4940$\pm $ 0.8945 & 0.4531 $\pm $ 1.0317    & 0.5696 $\pm $ 1.1828          & 0.5928 $\pm $ 1.6759  & 0.5366 $\pm $ 1.1849 & 0.5528 $\pm $ 1.0471  & 0.6394 $\pm $  1.0810   & 0.6065  $\pm$ 0.4640   & 0.5514 $\pm $ 0.9454       & 0.5988  $\pm $ 0.7499                           \\
& GBFRS & \textbf{0.5362$\pm $0.1333} & \textbf{0.5552 $\pm $ 0.0494 }  & \textbf{0.6026  $\pm $  0.0370}           & \textbf{0.7143 $\pm $ 0.0615 } & 0.5892 $\pm $ 0.0502 & \textbf{0.5613 $\pm $ 0.0398} & \textbf{0.6439 $\pm $ 0.0708}  &\textbf{ 0.7112 $\pm$ 0.0431}    & \textbf{ 0.6663  $\pm $  0.1811  } & \textbf{ 0.7004   $\pm $  0.0487 }       \\
                 & & & & & & & & & & &\\
\multirow{6}{*}{30\%} 
& FAR\_FIE & 0.4840  $\pm $1.5055  & 0.4653  $\pm $1.6098    & 0.5358  $\pm $ 1.4299          & 0.5422    $\pm $ 0.8217    &  0.5554   $\pm $ 0.8215  & 0.5481   $\pm $ 0.7923     & 0.5896   $\pm $ 1.3849   & 0.5576   $\pm $ 0.6279   & 0.4627   $\pm $  1.3697     & 0.4938    $\pm $ 1.3936                                                           \\
& FS\_NDEM & 0.2600  $\pm $1.0000   & 0.3265   $\pm $1.0756     & 0.5230   $\pm $ 1.2735          & 0.5431    $\pm $ 1.3355    &  0.4606   $\pm $ 0.8899  & 0.5372   $\pm $ 0.7602     & 0.5923   $\pm $ 0.4591   & 0.5237   $\pm $ 0.3599     & 0.5220   $\pm $  2.3499    & 0.5061    $\pm $ 0.3497                                                           \\
& FRDMAR   & 0.4880  $\pm $1.4832   & 0.3401   $\pm $1.4431     & 0.5290   $\pm $ 1.0426         & 0.5536    $\pm $ 0.6370    &  0.5549   $\pm $ 1.2020  & 0.5468   $\pm $ 0.7524     & 0.5810   $\pm $ 0.4312   & 0.5559   $\pm $ 1.0601     & 0.4678   $\pm $  0.2527    & 0.4995    $\pm $ 0.8167                                                           \\
& FNRS  & 0.1040$\pm $0.8944 & 0.3224 $\pm $ 0.7756  & 0.5618  $\pm $ 0.5340       & 0.5183 $\pm $1.4505   & \textbf{0.5931 $\pm $ 0.5570} & 0.4850 $\pm $ 1.0908 & 0.5292 $\pm $ 0.9850  & 0.6052 $\pm$ 0.8778    & 0.4497 $\pm $ 1.7229    & 0.4818  $\pm $  0.4841                                            \\
& HANDI & 0.5160$\pm $2.3022 & 0.4027 $\pm $  2.1187  & 0.5397 $\pm $ 0.8283       & 0.5183 $\pm $1.4357      & 0.5531 $\pm $ 0.9390 & 0.5507 $\pm $ 1.0932  & 0.5592 $\pm $ 0.2670 & 0.5824 $\pm$    1.2339 &  0.5401 $\pm $  0.5053   & 0.5761  $\pm $ 0.3945                                                             \\
& GBFRS & \textbf{0.5781$\pm $0.1443} & \textbf{0.4793 $\pm $ 0.0672} & \textbf{0.5893  $\pm $  0.0357 }          & \textbf{ 0.6381 $\pm $ 0.0596 }  & 0.5677 $\pm $ 0.0868 & \textbf{0.5522  $\pm $ 0.0777}  & \textbf{0.5978$\pm $ 0.0543 }& \textbf{0.6465 $\pm$   0.0375}  &  \textbf{0.5802   $\pm $  0.0737 } &  \textbf{ 0.6465   $\pm $ 0.0242  }      \\ \hline
\end{tabular*}} 
\end{table*}

	\begin{figure}[!h]
	\centering
	\includegraphics[width=3in]{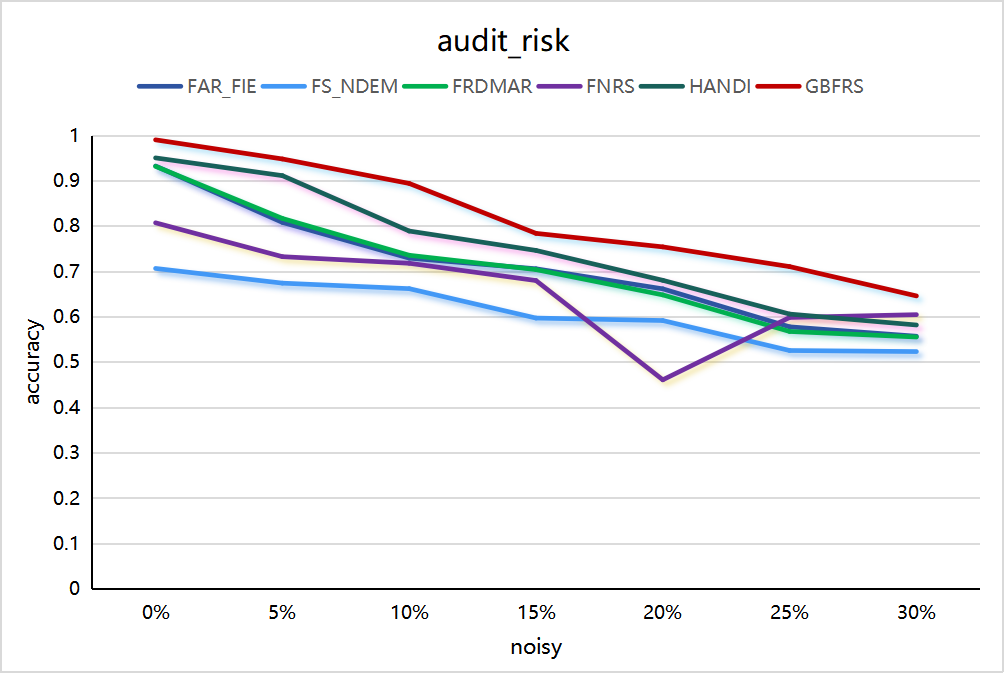}\\
	(a) Comparison of the accuracy of different methods on the audit\_risk dataset with different noise levels\\
	\includegraphics[width=3in]{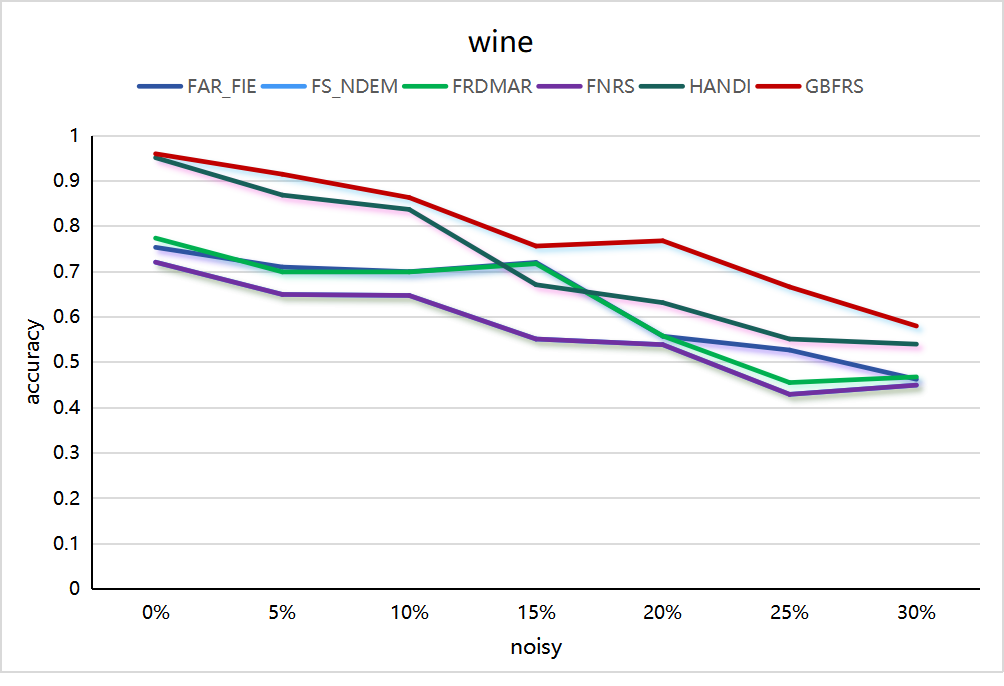}\\
	  (b) Comparison of the accuracy of different methods on the wine dataset with different noise levels
	\caption{Taking the dataset ``wine, audit\_risk'' as an example, the accuracy comparison line chart of each method under different noises. }
	\label{acc}
\end{figure} 

 In the experiment, each dataset is augmented with varying percentages of label noise: 0\%, 5\%, 10\%, 15\%, 20\%, 25\%, and 30\%. Notably, noise is introduced only to the training set, while the test set remains noise-free. Table \ref{tab:2} presents the accuracy comparison of all methods using the kNN classifier. The most accurate algorithms for each dataset are highlighted in bold. At 0\% noise level, HANDI or FRDMAR algorithms generally achieve the best results. However, as label noise increased, GBFRS exhibits robustness and high accuracy across most datasets. This can be attributed to the coarser granularity size of granular-balls, which mitigates the influence of noisy points. The label of each granular-ball is determined by the majority within the ball, minimizing the impact of minority-label noise points and thereby enhancing lower approximation accuracy.
 
 Fig. \ref{acc} illustrates the accuracy variations of different methods across different noise levels on the `wine' and `audit\_risk' datasets. It is evident that accuracy decreases with increasing noise for all methods. Nonetheless, GBFRS maintains the highest accuracy even under high noise conditions, underscoring the effectiveness of the proposed algorithm.


\begin{table*}[!h]
\setlength{\tabcolsep}{3pt}
\caption{Comparison of the accuracy of different methods for datasets with added attribute noise}
\label{tab:3}
\renewcommand\arraystretch{1.5}
\fontsize{5.8}{5.8}\selectfont
\setlength{\tabcolsep}{1.45mm}{ 
\begin{tabular*}{\linewidth}{@{}ccccccccccc@{}} 
\hline
Dataset & zoo                         & lymphography                & primary-tumor                                  & backup-large                & iono                        & Diabetes                    & wdbc                        & audit\_risk                 & wine                        & Parkinson\_Multiple         \\ \hline
FAR\_FIE    & 0.3828  $\pm $1.7807            & 0.4582  $\pm $2.1527            & 0.5270  $\pm $1. 3998                   & 0.7561   $\pm $0.9806           & 0.6945   $\pm $1.1967            & 0.5735   $\pm $0.4808            & 0.7478   $\pm $0.5125          & 0.5214   $\pm $0.9111            & 0.5835   $\pm $1.3135             & 0.5294    $\pm $0.7657           \\
FS\_NDEM    & 0.5345 $\pm $1.2191           & 0.7767 $\pm $0.9189           & 0.6314 $\pm $1.2104                  & 0.8053  $\pm $0.8976          & 0.6931  $\pm $0.8716           & \textit{ 0.6833  $\pm $0.3761   }        & 0.8847  $\pm $0.4247          & 0.5115  $\pm $0.6890           & 0.7148  $\pm $1.9435            & 0.5137   $\pm $0.4488           \\

FRDMAR    & 0.3690 $\pm $1.9658           & 0.4836 $\pm $2.7653           & 0.5255 $\pm $1.6923               & 0.7482  $\pm $1.8852          & 0.7017  $\pm $1.3841           & 0.5815  $\pm $0.9664           & 0.7566  $\pm $0.3741          & 0.5149  $\pm $0.9991           & 0.5909  $\pm $1.7045            & 0.5260   $\pm $0.5407           \\
HANDI   &\textit{ 0.9103$\pm $0.7711 }         & \textit{ 0.8082$\pm $1.2814 }         & \textit{ 0.6526$\pm $1.4461 }           & \textit{ 0.9794$\pm $0.1486 }         & \textbf{0.8615$\pm $0.7155} & 0.6830$\pm $0.8777          & \textbf{0.9263$\pm $0.4211} & \textbf{0.9524$\pm $0.3851} & \textit{ 0.9159$\pm $0.4754 }         & \textit{ 0.8256$\pm $0.3047 }         \\
GBFRS   & \textbf{0.9121$\pm $0.1114} & \textbf{0.8138$\pm $0.0393} & \textbf{0.6933$\pm $0.0369}  & \textbf{0.9833$\pm $0.0167} & \textit{0.8613$\pm $0.0722 }         & \textbf{0.6945$\pm $0.0398} & \textit{0.9118$\pm $0.0320 }         & \textit{0.9458$\pm $0.0226 }         & \textbf{0.9200$\pm $0.0619} & \textbf{0.8806$\pm $0.0446} \\ \hline
\end{tabular*}}
\end{table*}

In addition to label noise, another form of noise known as attribute noise involves perturbations applied to all attributes. Assuming a perturbation rate of 0.1, we conduct experiments to assess the performance of various algorithms on datasets with randomly perturbed attributes. After attribute reduction, the classification uses kNN classifier. Table \ref{tab:3} highlights the best results in bold and the second-best results in italics. GBFRS achieves optimal results for most datasets, as shown in Table \ref{tab:3}. While on three datasets GBFRS shows slightly inferior performance compared to the HANDI algorithm, these differences are minimal and all results were second-best. These experimental findings indicate that replacing sample points with granular-balls of different granularity levels can enhance model robustness against attribute noise.

\section{Conclusions}\label{sec6}

This paper introduces granular-ball computing into fuzzy rough sets and proposes a granular-ball fuzzy rough set framework, redefining the upper and lower approximations of GBFRS. Using granular-balls to replace sample points reduces the influence of noise points and improves the robustness of fuzzy rough sets. In addition, based on granular-balls, we re-fuzzify the similarity relationship and then use the lower approximation to define the weighted fuzzy dependency to achieve attribute reduction. Its relevant convergence is proven. Compared with other feature selection methods on the UCI dataset, we prove that the proposed model can achieve better accuracy in most cases. At present, granular-ball computing is not fully adaptive due to the threshold parameter. Therefore, we will also explore performance optimization to speed up the algorithm and improve running efficiency. In addition, another future work will focus on extending our proposed framework to other fuzzy rough set models to improve their performance.

\subsection{Acknowledgments}
This work was supported in part by the National Natural Science Foundation of China under Grant Nos. 62222601, 62221005, 62176033 and 61936001, Key Cooperation Project of Chongqing Municipal Education Commission under Grant No. HZ2021008, and Natural Science Foundation of Chongqing under Grant (cstc2022ycjh-bgzxm0128, cstc2021ycjh-bgzxm0013 and CSTB2023NSCQ-JQX0034).

\bibliographystyle{abbrv}
\bibliography{references}

\end{document}